\documentclass[a4paper, UKenglish, cleveref, autoref, thm-restate]{lipics-v2019}

\title{Maps for Learning Indexable Classes}

\titlerunning{Maps for Learning Indexable Classes}

\author{Julian Berger\footnote{firstname.lastname@student.hpi.uni-potsdam.de, with ``\"{o}'' as ``oe'' and double surnames written without space.}}{Hasso Plattner Institute, University of Potsdam, {Potsdam, Germany}}{}{}{}
\author{Maximilian B\"{o}ther\footnotemark[1]}{Hasso Plattner Institute, University of Potsdam, {Potsdam, Germany}}{}{}{}
\author{Vanja Dosko\v{c}\footnote{firstname.lastname@hpi.de, with ``\v{c}'' as ``c'' and ``\"{o}'' as ``oe''.}}{Hasso Plattner Institute, University of Potsdam, {Potsdam, Germany}}{}{}{}
\author{Jonathan Gadea Harder\footnotemark[1]}{Hasso Plattner Institute, University of Potsdam, {Potsdam, Germany}}{}{}{}
\author{Nicolas Klodt\footnotemark[1]}{Hasso Plattner Institute, University of Potsdam, {Potsdam, Germany}}{}{}{}
\author{Timo K\"{o}tzing\footnotemark[2]}{Hasso Plattner Institute, University of Potsdam, {Potsdam, Germany}}{}{}{}
\author{Winfried L\"{o}tzsch\footnotemark[1]}{Hasso Plattner Institute, University of Potsdam, {Potsdam, Germany}}{}{}{}
\author{Jannik Peters\footnotemark[1]}{Hasso Plattner Institute, University of Potsdam, {Potsdam, Germany}}{}{}{}
\author{Leon Schiller\footnotemark[1]}{Hasso Plattner Institute, University of Potsdam, {Potsdam, Germany}}{}{}{}
\author{Lars Seifert\footnotemark[1]}{Hasso Plattner Institute, University of Potsdam, {Potsdam, Germany}}{}{}{}
\author{Armin Wells\footnotemark[1]}{Hasso Plattner Institute, University of Potsdam, {Potsdam, Germany}}{}{}{}
\author{Simon Wietheger\footnotemark[1]}{Hasso Plattner Institute, University of Potsdam, {Potsdam, Germany}}{}{}{}

\authorrunning{Berger et al.}
\Copyright{Berger et al.}
\category{}
\relatedversion{}
\supplement{}
\nolinenumbers
\hideLIPIcs
\acknowledgements{This work was supported by DFG Grant Number KO 4635/1-1.}

\usepackage{times}

\usepackage{mathtools}

\usepackage[utf8]{inputenc}
\usepackage{amsopn}
\usepackage{amsfonts}
\usepackage{mathrsfs}
\usepackage{xparse}
\usepackage{upgreek}

\usepackage{wasysym}
\usepackage{tabu}

\usepackage{float}

\usepackage{verbatim}

\usepackage{adjustbox}

\newcommand*{\falls}{\text{if }}
\newcommand*{\sonst}{\text{otherwise}}
\newcommand*{\sonstfalls}{\text{else, if }}
\newcommand*{\cIf}{\text{if }}

\newcommand*{\otw}{\text{otherwise}}

\usepackage{thm-restate}
\newcounter{countprobl}
\newtheorem{problem}[countprobl]{Open Problem}

\usepackage{xparse}
\DeclareDocumentCommand{\set}{m g o}%
{%
    \IfNoValueTF{#3}{\left}{#3}\{#1
            \IfNoValueTF{#2}{}{\ \IfNoValueTF{#3}{\left}{#3}\vert\ \vphantom{#1}#2\IfNoValueTF{#3}{\right.}{}}
                \IfNoValueTF{#3}{\right}{#3}\}%
}

\DeclareDocumentCommand{\abs}{m o}%
{%
    \IfNoValueTF{#2}{\left}{#2}\vert#1
                \IfNoValueTF{#2}{\right}{#2}\vert%
}
\usepackage{hyperref}
\usepackage[capitalise, noabbrev]{cleveref}

\newcommand{\sort}{\operatorname{sort}}

\usepackage{adjustbox}

\newcommand{\CalL}{\mathcal{L}}
\newcommand{\natnum}{\mathbb{N}}

\newcommand*{\Ha}{\mathcal{H}}

\newcommand{\itemin}[1]{\item[#1\hspace{-0.8cm}] \hspace{0.8cm}}

\newcommand*{\IndF}{_{\mathbf{ind}}}

\newcommand*{\IndR}{}

\newcommand*{\concat}{^\frown}

\newcommand*{\Ind}{\mathbf{ind}}

\usepackage{tikz}
\usepackage{xspace}
\usetikzlibrary{shapes, fit, decorations.pathreplacing, shadows, arrows, calc, positioning}

\usepackage{pgfplots}
\usepgfplotslibrary{patchplots}

\newcommand*{\Txt}{\mathbf{Txt}}

\newcommand*{\G}{\mathbf{G}}

\newcommand*{\It}{\mathbf{It}}
\newcommand*{\Sd}{\mathbf{Sd}}
\newcommand*{\Psd}{\mathbf{Psd}}
\newcommand*{\Td}{\mathbf{Td}}

\newcommand*{\Ex}{\mathbf{Ex}}
\newcommand*{\Bc}{\mathbf{Bc}}

\newcommand*{\Fin}{\mathbf{Fin}}

\newcommand*{\Cons}{\mathbf{Cons}}
\newcommand*{\Conv}{\mathbf{Conv}}
\newcommand*{\Caut}{\mathbf{Caut}}
\newcommand*{\CautTar}{\Caut_{\textup{\textbf{Tar}}}}

\newcommand*{\True}{\mathbf{T}}
\newcommand*{\Wb}{\mathbf{Wb}}
\newcommand*{\NU}{\mathbf{NU}}
\newcommand*{\SNU}{\mathbf{SNU}}
\newcommand*{\Dec}{\mathbf{Dec}}
\newcommand*{\SDec}{\mathbf{SDec}}
\newcommand*{\WMon}{\mathbf{WMon}}
\newcommand*{\Mon}{\mathbf{Mon}}
\newcommand*{\SMon}{\mathbf{SMon}}

\newcommand*{\Sem}{\mathbf{Sem}}

\newcommand*{\SemConv}{\Sem\Conv}

\newcommand*{\Cind}{\mathbf{CInd}}
\newcommand*{\T}{\mathbf{T}}

\newcommand*{\N}{\mathbb{N}}

\newcommand*{\La}{\mathcal{L}}

\newcommand*{\Seq}{{\mathbb{S}\mathrm{eq}}}

\newcommand*{\totalCp}{\mathcal{R}}
\newcommand*{\partialCp}{\mathcal{P}}

\newcommand{\range}{\mathrm{range}}
\newcommand{\content}{\mathrm{content}}

\newcommand{\ind}{\mathrm{ind}}
\newcommand{\pad}{\mathrm{pad}}

\newcommand{\settwo}[1]{
  \{ #1 \}
}

\makeatletter
\newsavebox{\@brx}
\newcommand{\llangle}[1][]{\savebox{\@brx}{\(\m@th{#1\langle}\)}%
  \mathopen{\copy\@brx\kern-0.5\wd\@brx\usebox{\@brx}}}
\newcommand{\rrangle}[1][]{\savebox{\@brx}{\(\m@th{#1\rangle}\)}%
  \mathclose{\copy\@brx\kern-0.5\wd\@brx\usebox{\@brx}}}
\makeatother

\newcommand{\convs}{\mathclose{\hbox{$\downarrow$}}}
\newcommand{\divs}{\mathclose{\hbox{$\uparrow$}}}

\pgfdeclarelayer{background}
\pgfdeclarelayer{foreground}
\pgfsetlayers{background,main,foreground}

\begin{document}

\maketitle

\begin{abstract}
    We study learning of indexed families from positive data where a learner can
    freely choose a hypothesis space (with uniformly decidable membership)
    comprising at least the languages to be learned. This abstracts a very
    universal learning task which can be found in many areas, for example
    learning of (subsets of) regular languages or learning of natural languages.
    We are interested in various restrictions on learning, such as consistency,
    conservativeness or set-drivenness, exemplifying various natural learning
    restrictions.

    Building on previous results from the literature, we provide several maps
    (depictions of all pairwise relations) of various groups of learning
    criteria, including a map for monotonicity restrictions and similar criteria
    and a map for restrictions on data presentation. Furthermore, we consider,
    for various learning criteria, whether learners can be assumed consistent.
\end{abstract}

\section{Introduction}

We are interested in the problem of algorithmically learning a description for a
formal language (a computably enumerable subset of the set of all natural
numbers) when presented successively all and only the elements of that language;
this is called \emph{inductive inference}, a branch of (algorithmic) learning
theory. For example, a learner $h$ might be presented more and more even
numbers. After each new number, $h$ outputs a description for a language as its
conjecture. The learner $h$ might decide to output a program for the set of all
multiples of $4$, as long as all numbers presented are divisible by~$4$. Later,
when $h$ sees an even number not divisible by $4$, it might change this guess to
a program for the set of all multiples of~$2$.

Many criteria for determining whether a learner $h$ is \emph{successful} on a
language~$L$ have been proposed in the literature. Gold, in his seminal
paper~\cite{Gold67}, gave a first, simple learning criterion,
\emph{$\Txt\G\Ex$-learning}\footnote{$\Txt$ stands for learning from a
\emph{text} of positive examples; $\G$ stands for \emph{Gold-style} learning and
indicates that the learner has full information on the data given; $\Ex$ stands
for \emph{explanatory}.}, where a learner is \emph{successful} if and only if,
on every \emph{text} for $L$ (listing of all and only the elements of $L$) it
eventually stops changing its conjectures, and its final conjecture is a correct
description for the input language.  Trivially, each single, describable
language $L$ has a suitable constant function as a $\Txt\G\Ex$-learner (this
learner constantly outputs a description for $L$). Thus, we are interested in
analyzing for which \emph{classes of languages} $\CalL$ is there a \emph{single
learner} $h$ learning \emph{each} member of $\CalL$. This framework is also
known as \emph{language learning in the limit} and has been studied extensively,
using a wide range of learning criteria similar to $\Txt\G\Ex$-learning (see,
for example, the textbook~\cite{JORS99}).

A major branch of this analysis focuses on learning \emph{indexed families},
that is, classes of languages~$\CalL$ such that there is an enumeration
$(L_i)_{i \in \natnum}$ of all and only the elements of $\CalL$ for which the
decision problem ``$x \in L_i$'' is decidable. Already for such classes of
languages we get a rich structure. A survey of previous work in this area can be
found in~\cite{LZZ08}. We are specifically interested in \emph{class comprising}
learning, that is, our learners are free to choose any hypothesis space which
contains hypotheses at least for the languages to be learned. This is in
contrast, for example, to learning with a concretely given hypothesis space.

Since the appearance of the mentioned survey, only little work on indexable
classes was conducted, while learning of arbitrary families of languages
sprouted a new mode of analysis, \emph{map charting}. This approach tries to
further the understanding of learning settings by looking at all pairwise
relations of similar learning criteria and displaying them as a
\emph{map}~\cite{KP16,KS16}. This approach builds on the pairwise relations
which are already known in the literature and completes them in interesting
settings to understand one aspect more closely, for example regarding certain
natural restrictions on what kind of mind changes are allowed (which we will
consider in Section~\ref{Sec:DelMap}) or the importance of data presentation
(which we will consider in Section~\ref{Sec:tCindConvergence}).

We start our analysis by considering the restriction of \emph{consistent}
learning~\cite{Angluin80}. A learner is consistent if and only if each of its
hypotheses correctly reflects the data which the hypothesis is based on. Note
that, for arbitrary learning in the Gold-style model, learners cannot be assumed
consistent in general~\cite{Barzdin77}, a result termed the \emph{inconsistency
phenomenon}. The reason behind this result is essentially the same as for the
halting problem: a general hypothesis cannot be checked for consistency in a
computable way. Since, for indexed families, consistency of hypotheses is
decidable, it comes at no surprise that here learners can, in general, be
assumed consistent \cite{LZ95}. However, to prove this result, crucial changes
to the hypotheses are made (so as to make them consistent), which might spoil
other nice properties the learner might exhibit (such as never overgeneralizing
the true target language). In Section~\ref{Sec:Cind-Cons} we show several
different ways in which total learners can be made consistent, each maintaining
other restrictions (such as, for example, \emph{conservative}
learning~\cite{Angluin80}, where learners must not change their mind while still
consistent).

In Section~\ref{Sec:DelMap} we consider one of the best-studied maps from other
learning settings, the map of \emph{delayable} learning restrictions. We build
on previously known results, such as that conservative learning is
restrictive~\cite{LZ93-DepHypSp}, and complete the map both for the case of
\emph{full information} (where the learner has access to the full history of
data shown) and for \emph{set-driven} learners (which only have access to the
set of data presented so far, but not to the order of presentation~\cite{WC80}).
This builds on earlier analyses of monotone learning which has been studied in
various settings~\cite{LZ96}. We depict our results in
Figure~\ref{fig:tauCindMapGSd}. In particular, we show that the criteria cluster
into merely five different learning powers, i.e., many learning criteria allow
for learning the same classes of languages. Among other things, we show that
\emph{(strong) non-U-shaped} learning~\cite{BCMSW08,CM11}, where abandoning a
correct hypothesis is forbidden, is not restrictive in either setting
(set-driven and full-information).

\begin{figure}[h]
    \begin{center}
    \begin{adjustbox}{width=0.72\textwidth}
      \includegraphics{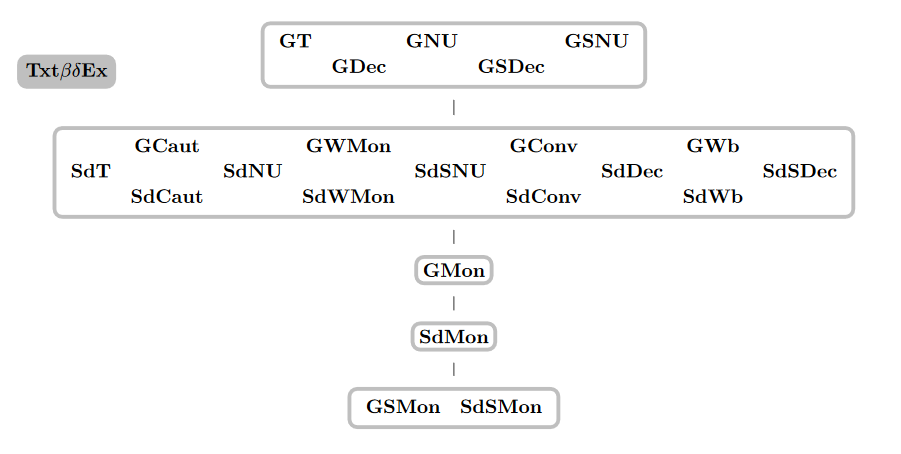}
    \end{adjustbox}
\end{center}
\caption{Relation of $\G$- and $\Sd$-learners under various additional restrictions, see Section~\ref{sub:mathematical_notation} for a full list thereof. The solid lines imply (proper) inclusions (bottom-to-top) and the greyly edged areas illustrate a collapse of the enclosed learning criteria.} \label{fig:tauCindMapGSd}
\end{figure}

In Section~\ref{Sec:tCindConvergence} we consider in more detail what impact the
access to information has on the learning power of total learners. Additionally
to full information and set-driven learning, we also consider \emph{iterative}
learning (where the learner has access to its previous hypothesis, but only the
current datum~\cite{WC80}). We give the complete map for $\Ex$-learning at the
same time as $\Bc$-learning (\emph{behaviorally correct} learning, where the
learner need not to stop syntactically changing the conjecture, so long as it
remains semantically correct~\cite{CL82,OW82}). We depict our findings in
Figure~\ref{fig:tauCindMemory}.

\begin{figure}[h]
\begin{center}
    \begin{adjustbox}{width=0.42\textwidth}
    \includegraphics{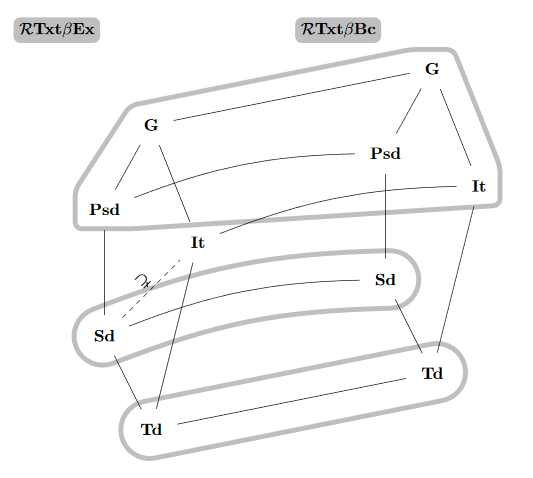}
    \end{adjustbox}
\end{center}
    \caption{Relation of syntactic and semantic convergence when learning indexable classes with total learners under various memory restrictions $\beta$. Black solid lines imply trivial inclusions (bottom-to-top, left-to-right). The dashed line depicts the non-trivial proper inclusion $[\totalCp\Txt\It\Ex]_\Ind \subsetneq [\totalCp\Txt\Sd\Ex]_\Ind$. Furthermore, greyly edged areas illustrate a collapse of the enclosed learning criteria and there are no further collapses.}
    \label{fig:tauCindMemory}
\end{figure}

These directions taken together (consistency, delayable restrictions,
information access, syntactic vs.\ semantic convergence) give a well-rounded
picture, offering a glimpse on learning indexable classes from all commonly
studied angles.

In order to prove our results, we develop a very useful characterization of
learning indexable families given in Theorem~\ref{thm:ind-fam}. Here we show
that learnability as an indexed family with an arbitrary hypothesis space is
equivalent to learnability by a learner which only outputs programs for
characteristic functions and is considered successful when converging to such a
program which decides the target language. This result allows us to simplify
many of our proofs (and it also made finding proofs easier), since now the
hypothesis space does not need to be chosen in advance.

We continue this paper with a section on mathematical preliminaries
(Section~\ref{sec:preliminaries}), including some relevant results from the
literature, before getting to the technical part.

\section{Preliminaries}%
\label{sec:preliminaries}

In this section we introduce the mathematical notations and notions used
throughout the paper. For unintroduced notation we refer to~\cite{Rogers87}.
Regarding the learning criteria, we follow the system of~\cite{Kotzing09}.

\subsection{Language Learning in the Limit}%
\label{sub:mathematical_notation}

We denote the set of natural numbers by \(\N=\set{0,1,2,\ldots} \). With
\(\subseteq\) and \(\subsetneq\) we denote the subset and proper
subset relation between sets, respectively. Furthermore, with
\(\subseteq_\textbf{FIN}\) we denote finite subsets. With \(\cap,\cup,
\setminus\) we denote the set intersection, union, and difference, respectively.
We let \(\emptyset\) and \(\varepsilon\) denote the empty set and empty
sequence, respectively. We use \(\partialCp\) (\(\totalCp\)) as the set of all
(total) computable functions. If a function \(f\) is defined on an argument
\(x\) we denote this by \(f(x)\convs\); otherwise, we write \(f(x)\divs\). We
fix an effective  numbering \({(\varphi_{e})}_{e\in\N}\) of \(\partialCp\), where
\(e\) may be viewed as a \emph{program} or \emph{index} for the
function~\(\varphi_{e}\).

We fix the symbol \(\#\) called \emph{pause}. For any set $S \subseteq \N$, we
denote $S_\# \coloneqq S \cup \{ \# \}$. The set of all sequences of length $t
\in \N$ over $S_\#$ is denoted by \(S^{\leq t}_\#\) and the set of all finite
sequences over \(\N\cup \set{\#} \) by \(\Seq\). For two sequences
\(\sigma,\tau\), we let \(\sigma\concat\tau\) denote their concatenation and we
write \(\sigma\subseteq \tau\) if and only if \(\sigma\) is a prefix of
\(\tau\). For a (possibly infinite) sequence \(\sigma\), we let
\(\content(\sigma)=(\range(\sigma) \setminus \set{\#})\). For \(\sigma\in\Seq\),
we denote the sequence with the last element removed as  \(\sigma^-\).
Furthermore, we may interpret finite sequences as natural numbers and fix a
total order $\leq$ on these such that, in particular, for all
\(\sigma,\tau\in\Seq\) with  \(\sigma\subseteq \tau\) we have that
\(\sigma\leq\tau\).

We call a computably enumerable set \(L\subseteq \N\) a \emph{language}. We
learn \emph{indexed families} of languages, that is, families of languages
$(L_i)_{i \in \N}$ where there exists a total computable function $f$ such that,
for all $i, x \in \N$,
\[
	f(i,x) = \begin{cases}
		1, & \falls x \in L_i; \\ 0, &\sonst.
	\end{cases}
\]
We learn these families with respect to hypothesis spaces, which are indexed
families of languages themselves. In general, a \emph{learner} is a function $h
\in \partialCp$.  We examine
learning from text. A \emph{text} is a total function \(T:\N\to\N\cup
\set{\#}\). We denote the set of all texts as \(\Txt\). Furthermore, we call
\(T\) a \emph{text for a language \(L\)} if \(\content(T)=L\); the set of all
texts for  \(L\) is denoted by \(\Txt(L)\). The \emph{canonical text} of a
language \(L\) is the enumeration of all elements in \(L\) in strictly ascending
order (if $L$ is finite, the text returns $\#$ after all elements have been
presented). Analogously, the \emph{canonical sequence} of a (finite) language
$L$ is the (finite) sequence of all elements in $L$ in strictly ascending order.
Additionally, we define \(T[0]=\varepsilon\) and, for all \(n\in\N\) with $n>0$,
\(T[n]=T(0),\ldots,T(n-1)\).

What kind of information a learner is given, is specified by an
\emph{interaction operator}. Formally, an interaction operator is a function
that takes a learner and a text as input arguments and outputs a (possibly
partial) function that is called \emph{learning sequence} or \emph{sequence
	of hypotheses}. We consider \emph{Gold-style} or \emph{full-information} learning~\cite{Gold67}, denoted by \(\G\), \emph{iterative} learning
(\(\It\),~\cite{Fulk85,Wiehagen76}), \emph{partially set-driven} or
\emph{rearrangement-independent} learning (\(\Psd\),~\cite{BlumBlum75,SchRicht84}),
\emph{set-driven} learning (\(\Sd\),~\cite{WC80}) and \emph{transductive}
learning (\(\Td\),~\cite{CCJS07,Kotzing09}). Note that transductive learners may
output a special symbol ``\(\mbox{?}\)'' if the information given is not
sufficient to make a guess. Formally, for all learners \(h\in\partialCp\), texts
\(T\in\Txt\) and \(i\in\N\),
\begin{align*}
	\G(h,T)(i)   & = h(T[i]);                  \\
	\Psd(h,T)(i) & = h(\content(T[i]),i);      \\
	\Sd(h,T)(i)  & = h(\content(T[i]));        \\
	\It(h,T)(i)  & = \begin{cases}
		h(\varepsilon),          & \cIf i=0; \\
		h(\It(h,T)(i-1),T(i-1)), & \otw;
	\end{cases} \\
	\Td(h,T)(i)  & =\begin{cases}
		\mbox{?},      & \cIf i=0;                       \\
		\Td(h,T)(i-1), & \sonstfalls h(T(i-1))=\mbox{?}; \\
		h(T(i-1)),     & \otw.
	\end{cases}
\end{align*}
Intuitively, Gold-style learners have full information on the elements given.
Set-driven learners base their hypotheses solely on the content of the
information given, while partially set-driven learners additionally have a
counter for the iteration step. Iterative learners base their conjectures on
their previous hypothesis and the current datum. Lastly, transductive learners
solely base their guesses on the current datum and may output ``?'' if the
information is not sufficient.

For two interaction operators \(\beta,\beta'\) we write \(\beta\preceq\beta'\)
if and only if every \(\beta\)-learner \(h\) can be compiled into an
equivalent \(\beta'\)-learner \(h'\) such that, for any text \(T\), we have
\(\beta(h,T)=\beta'(h',T)\). We note that  \(\Td\preceq\It\preceq\G\) and
\(\Sd\preceq\Psd\preceq\G\). As an example, every  \(\Sd\)-learner can be
compiled into an \(\Psd\)-learner by simply ignoring the counter. Furthermore,
note that any $\Td$-learner may be simulated by a $\Sd$-learner, however, the
order of the hypotheses may be changed. For any \(\beta\)-learner \(h\) with
\(\beta\preceq\G\), we let \(h^*\), the
\emph{starred learner}, denote the \(G\)-learner simulating \(h\). For example,
the starred learner of a \(\Psd\)-learner \(h\) is defined, for all sequences
\(\sigma\), as \(h^*(\sigma)=h(\content(\sigma),\abs{\sigma})\).

For a learner to successfully identify a language it has to satisfy certain
restrictions. A famous example was given by Gold, who required the
learner to converge to a correct
hypothesis for the target language~\cite{Gold67}. This is called \emph{explanatory} learning
and denoted by \(\Ex\). When we speak of correct hypotheses, it is
with regard to an indexed hypothesis space. Formally, a learning restriction is a predicate
on a sequence of
hypotheses \(p\) and a text \(T\in\Txt\). In the case of explanatory learning,
we get, for a given indexed hypothesis space \(\mathcal H ={(H_i)}_{i\in\N}\),
\begin{equation*}
	\Ex(p,T)\Leftrightarrow \exists n_0\forall n\geq n_0\colon
	p(n)=p(n_0)\land H_{p(n_0)}=\content(T).
\end{equation*}
We now give the intuition for the considered
restrictions and define them formally afterwards. As an alternative to \(\Ex\), for
\emph{behaviorally-correct} (\(\Bc\)) learning one only requires semantic
convergence, that is, after some point all hypotheses must be correct hypotheses
for the target language, but they do not need to be syntactically
equal~\cite{CL82,OW82}.

In addition to these convergence criteria there are various other
properties that are natural to require from a learner. In \emph{non-U-shaped}
learning (\(\NU\),~\cite{BCMSW08}), once the learner outputs a correct
hypothesis, it  may not unlearn the language, i.e., it may only make syntactic
mind changes. In \emph{strongly non-U-shaped} learning (\(\SNU\),~\cite{CM11})
not even these syntactic mind changes are allowed. In \emph{consistent}
learning (\(\Cons\),~\cite{Angluin80}), each hypothesis must include the given
information.
There exist various
monotonicity restrictions (\cite{Jantke91,Wiehagen91,LZ93}). When learning
\emph{strongly monotone} (\(\SMon\)), the learner may not discard elements
present in previous hypotheses, and in \emph{monotone} learning (\(\Mon\))
the learner is not allowed to remove correct data from its hypotheses.
Furthermore, in \emph{weakly monotone} learning (\(\WMon\))
the learner must remain strongly monotone while consistent with the
input. Similarly, in \emph{cautious} learning
(\(\Caut\),~\cite{OSW82}), no hypothesis may be a proper subset of a prior
hypothesis. As a relaxation, in \emph{target-cautious} learning
(\(\CautTar\),~\cite{KP16}), no hypothesis may be a proper superset of the
target language. A specialization of cautious and weakly monotone learning is
\emph{witness-based} learning (\(\Wb\),~\cite{KS16}), where each mind change
must be justified by a
witness. Witness-based learning is also a specialization of \emph{conservative} learning (\(\Conv\),~\cite{Angluin80}) where the learner may only make a mind
change when an inconsistency is detected.
If we only require this for semantic mind changes, we call the learner \emph{semantically
	conservative} (\(\SemConv\),~\cite{KSS17}). Finally, in \emph{decisive}
learning (\(\Dec\),~\cite{OSW82}), the learner may not return to semantically
abandoned hypotheses; in \emph{strongly decisive} learning
(\(\SDec\),~\cite{Kotzing17}), the learner may not return to syntactically
abandoned hypotheses. Now, we give formal definitions for theses restrictions.
Let \(\Ha = {(H_i)}_{i\in\N}\) be an indexed hypothesis space.
For any sequence of hypotheses \(p\) and text  \(T\in\Txt\), we have
\begin{align*}
	\Bc\IndR(p, T) & \Leftrightarrow\exists n_0\colon\forall n\geq
        n_0\colon H_{p(n)}=\content(T); \\
	\NU(p,T)       & \Leftrightarrow\forall i,j,k\colon i\leq j \leq k \land
        H_{p(i)}=H_{p(k)}=\content(T) \Rightarrow H_{p(i)}=H_{p(j)}; \\
	\SNU(p,T)      & \Leftrightarrow\forall i,j,k\colon i\leq j \leq k \land
        H_{p(i)}=H_{p(k)}=\content(T) \Rightarrow p(i)=p(j); \\
	\Cons(p,T)     & \Leftrightarrow\forall i\colon\content(T[i])\subseteq
        H_{p(i)};\\
	\SMon(p,T)     & \Leftrightarrow\forall i,j\colon i< j\Rightarrow
        H_{p(i)}\subseteq  H_{p(j)}; \\
	\Mon(p,T)      & \Leftrightarrow\forall i,j\colon i<
        j\Rightarrow \content(T)\cap H_{p(i)}\subseteq \content(T)\cap H_{p(j)}; \\
	\WMon(p,T)     & \Leftrightarrow\forall i,j\colon i<
        j\land\content(T[j])\subseteq H_{p(i)}\Rightarrow H_{p(i)}\subseteq
        H_{p(j)}; \\
	\Caut(p,T)     & \Leftrightarrow\forall i,j\colon H_{p(i)}\subsetneq
        H_{p(j)}\Rightarrow i\leq j; \\
	\CautTar(p,T)  & \Leftrightarrow\forall i\colon \neg(\content(T)\subsetneq
        H_{p(i)}); \\
	\Wb(p,T)       & \Leftrightarrow \forall i,j \colon \left( \exists k \colon i<k\leq j \wedge
        p(i)\neq p(k) \right) \Rightarrow \\
                   & \phantom{\Leftrightarrow \forall i,j \colon \exists k
                   \colon} \Rightarrow \left(\content(T[j])\cap H_{p(j)}\right)
                   \setminus H_{p(i)}\neq\emptyset; \\
	\Conv(p,T)     & \Leftrightarrow\forall i\colon \content(T[i+1])\subseteq
        H_{p(i)}\Rightarrow p(i)=p(i+1); \\
	\SemConv(p,T)  & \Leftrightarrow\forall i\colon \content(T[i+1])\subseteq
        H_{p(i)}\Rightarrow H_{p(i)}=H_{p(i+1)}; \\
	\Dec(p,T)      & \Leftrightarrow\forall i,j,k\colon i\leq j \leq k \land
        H_{p(i)}=H_{p(k)}\Rightarrow H_{p(i)}=H_{p(j)}; \\
	\SDec(p,T)     & \Leftrightarrow\forall i,j,k\colon i\leq j \leq k \land
	    H_{p(i)}=H_{p(k)}\Rightarrow p(i)=p(j).
\end{align*}
We combine any two learning restrictions \(\delta\) and \(\delta'\) by
intersecting them, which is denoted by their juxtaposition. With \(\True\) we
define the learning restriction which is always true and interpret it as absence
of a learning restriction.

Now, a \emph{learning criterion} is a tuple \((\alpha,\mathcal{C},\beta,\delta)\),
where \(\alpha\) and \(\delta\) are learning restrictions, \(\mathcal{C}\) is
the set of admissible learner, usually \(\partialCp\) or \(\totalCp\), and
\(\beta\) is an interaction operator. We write
\(\tau(\alpha)\mathcal{C}\Txt\beta\delta\) to denote the learning criterion and omit
\(\mathcal{C}\) if it equals \(\partialCp\), and a learning restriction if it
equals \(\True\). Let \(h\in\mathcal{C}\) be an admissible learner. We say that
\(h\) \(\tau(\alpha)\mathcal{C}\Txt\beta\delta\)-learns a language \(L\) with respect to some hypothesis space $\Ha$ if and only
if, for all texts \(T\in\Txt\), we have \(\alpha(\beta(h,T),T)\) and, for all
\(T\in\Txt(L)\), \(\delta(\beta(h,T),T)\). The set of languages
\(\tau(\alpha)\mathcal{C}\Txt\beta\delta\)-learned by \(h\) with respect to some hypothesis space $\Ha$ is denoted by
\(\tau(\alpha)\mathcal{C}\Txt\beta\delta(h)\). The set of all indexable families
\(\tau(\alpha)\mathcal{C}\Txt\beta\delta\)-learned by an admissible learner with
respect to some indexed hypothesis space is denoted by
\({[\tau(\alpha)\mathcal{C}\Txt\beta\delta]}_{\Ind}\), the so-called \emph{learning power} of $\tau(\alpha)\mathcal{C}\Txt\beta\delta$-learners.

\subsection{Normals Forms}%
\label{sub:normals_forms}

To proof certain statements on learner, there are properties that come in handy.
For example, except for $\Cons$, all considered learning restrictions are \emph{delayable}. Intuitively, a
learning restriction is delayable if it allows for arbitrary, but finite
postponing of hypotheses~\cite{KP16}. Formally, a learning restriction is delayable if and
only if for all sequences of hypotheses \(p\), texts \(T,T'\in\Txt\) with
\(\content(T)=\content(T')\) and non-decreasing, unbounded functions
\(r\colon\N\to\N\), if we have \(\delta(p,T)\) and, for all \(n\in\N\),
\(\content(T[r(n)])\subseteq \content(T'[n])\), then also \(\delta(p\circ r,
T')\) holds.

A common property of the considered learning restrictions is that they solely
depend on the semantic of the hypotheses and on the position of mind changes.
This property is formalized in the notion of \emph{pseudo-semantic} learning
restrictions~\cite{KSS17}. A learning restriction \(\delta\) is pseudo-semantic
if and only if for all learning sequences \(p\) and texts \(T\in\T\), if
\(\delta(p,T)\) and for a learning sequence $p'$, with, for all \( n \in\N\),
\(p(n)\) and \(p'(n)\) are semantically equivalent and \(p(n)=p(n+1)\) implies
\(p'(n)=p'(n+1)\), then \(\delta(p',T)\).  All considered learning restrictions
are  pseudo-semantic.

We regularly make use of \emph{locking sequences}. These are sequences that
contain enough information such that a given learner, after seeing this
sequence, suggests a correct hypothesis for the target language and does not
change its mind whatever data from the target language it is given. Formally,
let \(h\) be a \(\G\)-learner and $\Ha = (H_i)_{i \in \N}$ a hypothesis space.
Then a sequence \(\sigma\in\Seq\) is a locking sequence for \(h\) on a language
\(L\), if, for all sequences \(\tau\in L_\#^*\), we have
\(h(\sigma)=h(\sigma\concat\tau)\) and \(H_{h(\sigma\concat \tau)} =
L\)~\cite{BlumBlum75}. For $\Bc$-learners, we drop the first requirement and
call $\sigma$ a \emph{$\Bc$-locking sequence}~\cite{JORS99}. This definition can
directly be expanded to learners with other interaction operators. Let $h$ be
such a learner and consider its starred learner $h^*$. Then, a sequence $\sigma$
is called a locking sequence for $h$ on $L$ if and only if $\sigma$ is a locking
sequence for $h^*$ on $L$. We remark that, in the case of partially set-driven
and set-driven learners, we refer to locking sequences as \emph{locking
information} and \emph{locking set}, respectively.  Note that, if a learner
learns a language there always exists a ($\Bc$-) locking sequence
\cite{BlumBlum75}, but there exist texts where no initial sequence thereof is a
($\Bc$-) locking sequence. Given a learner $h$ and a language $L$ it learns, if
on any text $T \in \Txt(L)$ there exists an initial sequence thereof which is a
($\Bc$-) locking sequence for $h$ on $L$, we call $h$ \emph{strongly ($\Bc$-)
locking on $L$}. If $h$ is strongly ($\Bc$-) locking on every language it
learns, we call $h$ \emph{strongly ($\Bc$-) locking} \cite{KP16}.

\section{Learning Indexed Families without Hypothesis Spaces}\label{Sec:NoHype}

In this section we present a useful result on which we build our remaining
results. When learning indexed families with respect to (arbitrary) hypothesis
spaces, the choice of the latter is crucial for successful learning. However, it
is often a non-trivial task to construct the fitting hypothesis space. With
Theorem~\ref{thm:ind-fam}, we show that one can forgo this necessity and, so to
speak, obtain the hypothesis space on the run.

We make use of so-called \emph{$C$-indices} or \emph{characteristic indices}.
Intuitively, a $C$-index of a language $L$ is a program for its characteristic
function. Formally, an index $e$ is a $C$-index of the language $L$ if and only
if $\varphi_e \equiv \chi_L$. We also denote $C_e =
\set{x\in\N}{\varphi_{e}(x)=1 }$. Note that if $e$ is a $C$-index of $L$ then
$C_e = L$. Now, we can request a learner to converge to a $C$-index instead of
an index with respect to some hypothesis space. Exemplary, when requiring
syntactic convergence to a $C$-index we write, for all sequences of hypotheses
$p$ and all texts $T$, \[
	\Ex_C \Leftrightarrow \exists n_0\forall n\geq n_0\colon
	p(n)=p(n_0)\land C_{p(n_0)}=\content(T).
\]
Transitioning the other considered restrictions is immediate and, thus, omitted.
For clarity, given a learning criterion $(\alpha,\mathcal{C},\beta,\delta)$, we
write $\tau(\alpha)\mathcal{C}\Txt\beta\delta_C$ in case of learning
$C$-indices. Analogously, for example, we denote with
$[\tau(\alpha)\mathcal{C}\Txt\beta\delta_C]$ the set of all classes
$\tau(\alpha)\mathcal{C}\Txt\beta\delta_C$-learnable by some learner $h$.

In particular, we show that learners which output characteristic indices on any
input may be translated into total learners which learn with respect to a
hypothesis space. To that end, we define the restriction $\Cind$, where the
learner must output $C$-indices. Formally, for any hypothesis sequence $p$ and
any text $T$, we have
\[
	\Cind(p,T) \Leftrightarrow\forall i,x\colon
	\varphi_{p(i)}(x)\convs \wedge \varphi_{p(i)}(x)\in\set{0,1}.
\]
We show the equality of the two learning approaches. While the output of $h'$
can easily be interpreted as a characteristic index, for the other direction,
one considers all hypotheses output by the $\tau(\Cind)$-learner, that is, the
learner which outputs $C$-indices on any input, as hypothesis space. Then, it
remains to choose the right (minimal) index of this hypothesis space to maintain
successful learning. We provide the rigorous proof.

\begin{theorem}%
	\label{thm:ind-fam}\label{Thm:ind-fam}
	Let $\alpha, \delta$ be pseudo-semantic restrictions, let $\beta \preceq \G$ be an
	interaction operator and let $\La$ be an indexed family. Then, \(\mathcal{L}\) is in
	\([\tau(\Cind\alpha)\Txt\beta\delta_C]\) if and only if there exist a total
	learner \(h'\) and a hypothesis space \(\mathcal{H}\) such that \(h'\)
	\(\tau(\alpha)\Txt\beta\delta\)-learns \(\mathcal{L}\) with respect
	to ${\mathcal{H}}$.
\end{theorem}

\begin{proof}
	For the first direction, let $\La \in {[\tau(\Cind\alpha)\Txt\beta\delta_C]}$
	be an indexed family learned by a
	\(\tau(\Cind\alpha)\Txt\beta\delta_C\)-learner
    \(h\). Let $h^*$ be the starred form of $h$, that is, the $\G$-learner
    simulating $h$. As $h$ is $\tau(\Cind)$, so is $h^*$ and we can define the
    indexed hypothesis space $\Ha = {(C_{h^*(\sigma)})}_{\sigma \in \Seq}$. As
    $h$ learns $\La$, we have $\La \subseteq \Ha$.
	Fix an order $\leq$ on the set of all finite sequences. Then, we define the
	learner $h'$, for notational convenience in its starred form, as, for any
	finite sequence $\sigma$,
	\[
		{(h')}^*(\sigma) = \min_\leq\{ \sigma'\in\Seq\mid h^*(\sigma') = h^*(\sigma) \}.
	\]
	Note that the \(\min\)-search terminates as $\sigma$ is a candidate thereof.
	Now, for two sequences $\sigma$ and $\tau$, we have $h^*(\sigma) = h^*(\tau)$ if and
    only if ${(h')}^*(\sigma) = {(h')}^*(\tau)$. Also, \(h^*(\sigma)\) and
    \({(h')}^*(\sigma)\) are semantically equivalent. Thus, $h'$
	$\tau(\alpha)\Txt\beta\delta$-learns $\La$ with respect to $\Ha$.

	Conversely, let \(\mathcal{L}\) be such that there exist a total learner \(h'\)
	and an indexed hypothesis space \(\mathcal{H} = {(L_j)}_{j \in \N}\) such that
	\(h'\) \(\tau(\alpha)\Txt\beta\delta\)-learns \(\mathcal{L}\) with
	respect to $\Ha$. We
	provide a learner $h$ which
	\(\tau(\Cind\alpha)\Txt\beta\delta_C\)-learns $\La$. Let \({(h')}^*\) and \(h^*\)
	denote their starred forms. As
	$\Ha$ is an indexed hypothesis space, there exists a total computable function
	$f$ such that for all $j, x \in \N$
	\[
		f(j, x) = \begin{cases} 
            1, &x \in L_j; \\ 
            0, &\sonst. 
        \end{cases}
	\]
	Due to the S-m-n Theorem, there exists a strictly monotonically increasing
	function $g \in \totalCp$ such that, for all $j, x \in \N$, we have
	$\varphi_{g(j)}(x) = f(j, x)$. Now, we define, for all
	finite sequences $\sigma\in\Seq$,
	\[
		{h}^*(\sigma) = g((h')^*(\sigma)).
	\]
	We conclude the proof by showing that $h$
	$\tau(\Cind\alpha)\Txt\beta\delta_C$-learns $\La$. We first show that $h$ is
	$\tau(\Cind)$. This follows directly as, for any finite sequence $\sigma$,
	there exists \(j\in\N\) such that
	\[
		\varphi_{h(\sigma)}(x) = \varphi_{g(j)}(x) = f(j,x) = \begin{cases} 
            1, & x \in L_{j}; \\ 
            0, &\sonst.
        \end{cases}
	\]
	As $h'$ only makes mind changes when $h$ does and as, for any $\sigma\in\Seq$,
	$L_{(h')^*(\sigma)} = C_{{h}^*(\sigma)}$, we have that $h$
	$\tau(\alpha)\Txt\beta\delta_C$-learns $\La$.
\end{proof}

We will see that in many cases requiring the learner $h'$ to be total is no
restriction. It is already known that, when learning arbitrary classes of
recursively enumerable languages, Gold-style learners, obeying delayable
learning restrictions, may be assumed total \cite{KP16}. This result directly
transfers to learning indexed families with respect to some hypothesis space.
The following theorem holds.

\begin{theorem}%
	\label{thm:g-wlog-total}
	Let \(\delta\) be a delayable learning restriction. Then, we have that
	\begin{align*}
		{[\totalCp\Txt\G\delta]}\IndF={[\Txt\G\delta]}\IndF.
	\end{align*}
\end{theorem}

\begin{proof}
	This proof follows~\cite{KP16}. The inclusion
	\({[\totalCp\Txt\G\delta]}\IndF \subseteq {[\Txt\G\delta]}\IndF\)
	is immediate. For the other, let \(h\)
    $\Txt\G\delta$-learn $\La$ with respect to a hypothesis space $\Ha$. Let
    \(e\in\N\) such that \(h=\varphi_{e}\). To define the equivalent learner,
    let \(\Phi\) be a Blum complexity measure~\cite{Blum67}, that is for
    example, for \(e,x\in\N\),
	\(\Phi_e(x)\) could be the number of steps the program \(e\) needs to halt on input
	\(x\). We define a learner \(h'\)
	such that, for all sequences \(\sigma\in\Seq\),
	\begin{align*}
		h'(\sigma)=h(\max_\subseteq(
		\set{\sigma'\subseteq\sigma}{\Phi_e(\sigma')\leq\abs{\sigma}}\cup
		\set{\varepsilon})).
	\end{align*}
	As we only allow total learning sequences of \(h\) for languages in \(\La\),
	we have \(h(\varepsilon)\convs\) and, thus, \(h'\) is indeed total and
    computable. We show that $h$ \(\totalCp\Txt\G\delta\)-learns $\La$ with
    respect to $\Ha$. To that end, we use that \(\delta\) is delayable. Let $L
    \in \La$ and \(T\in\Txt(L)\). Now, for all
	\(n\in\N\),  let \(r(n)=\abs{\max_\subseteq(
		\set{\sigma'\subseteq T[n]}{\Phi_e(\sigma')\leq n}\cup
		\set{\varepsilon})}\). Note that, for
    all \(n\in\N\), we have \(h'(T[n])=h(T[r(n)])\). As $\delta$ is delayable,
    it suffices to show that $r$ is non-decreasing and unbounded to prove that
    \(h'\) \(\totalCp\Txt\G\delta\)-learns \(\La\) with respect to $\Ha$. By
    definition of \(r\), we have that \(r\) is
	non-decreasing and, for all \(n\in\N\), we have \(r(n)\leq n\) and that
	\(r\) is unbounded, as there exists  \(m\in\N\) with \( m\geq n\)
	such that \(\Phi_e(T[n])\leq m\) and, thus \(r(m)\geq n\).
	This concludes the proof.
\end{proof}

\section{Learning Indexable Classes Consistently}\label{Sec:Cind-Cons}

Learners may have various useful properties. One such is being consistent with
the information given while maintaining learning power. For example, various
behaviorally correct learners have been investigated for consistency
\cite{KSS17}. We study whether this can also be assumed when learning indexed
families and also with explanatory learners. Throughout this section, we provide
the individual results which, gathered together, provide the following theorem.

\begin{theorem}\label{th:taucons}
	For all $\delta \in \{ \True, \Mon, \SMon, \WMon, \CautTar, \SemConv, \Conv
		\}$ and all $\delta' \in \{ \Ex, \Bc \}$ as well as all \(\beta \in \{\G,
	\Psd, \Sd\}\), we have
	\[
        {[\tau(\Cons)\Txt\beta\delta{\delta'}\IndR]}\IndF =
		{[\totalCp\Txt\beta\delta{\delta'}\IndR]}\IndF.
    \]
\end{theorem}

Unrestricted $\Bc$-learners can be made consistent by simply patching in the
missing elements into the hypothesis. As one can check for consistency, one can
decide whether changing the hypothesis is necessary or not. It is immediate to
see that this strategy works out, as the padding needs only to be done while the
learner did not converge yet and as the learner needs not to serve any
additional requirements. Note that this also preserves $\Ex$-convergence.
Interestingly, the same idea also works out for certain restricted learners. In
particular, strongly monotone and monotone learners can be made consistent this
way as well. Especially here, Theorem~\ref{thm:ind-fam} comes in handy as we do
not need to fix the hypothesis space containing the padded hypotheses
beforehand. We provide the result.

\begin{lemma}
	For \(\beta \in \{\G, \Psd, \Sd\}\), \(\delta \in \{\T,\Mon,\SMon\}\)
	and $\delta' \in \{ \Ex, \Bc \}$, we have
	\[
        {[\tau(\Cons)\Txt\beta\delta{\delta'}]}\IndF =
		{[\totalCp\Txt\beta\delta{\delta'}]}\IndF.
    \]
\end{lemma}

\begin{proof}
	The inclusion \({[\tau(\Cons)\Txt\beta\delta{\delta'}\IndR]}\IndF \subseteq
	{[\totalCp\Txt\beta\delta{\delta'}\IndR]}\IndF\)
	is immediate. For the other,
	we use a construction which patches in the seen data while maintaining the
    given learning restriction, as seen in \cite{KSS17} for learning of
    arbitrary classes. By \cref{Thm:ind-fam}, it suffices to show
	\[
        {[\tau(\Cind)\Txt\beta\delta\delta_C']}\subseteq
		{[\tau(\Cind\Cons)\Txt\beta\delta\delta_C']}.
    \] Let
	\(h\) be a learner and let $\La\subseteq
		\tau(\Cind)\Txt\beta\delta\delta'_C(h)$. Using some auxiliary functions, we define
    a $\tau(\Cind\Cons)\Txt\beta\delta{\delta'_C}$-learner $h'$. For ease of
    notation, we use $h$ and $h'$ as their
	starred learners. Due to the S-m-n Theorem, there exists \(s \in \totalCp\) such
	that, for all \(x \in \N\) and all finite sequences \(\sigma\),
	\begin{align*}
		\varphi_{s(\sigma)}(x) = \begin{cases}
			1, & \falls x \in \content(\sigma) \vee \varphi_{h(\sigma)}(x) = 1; \\
			0, & \otw.
		\end{cases}
	\end{align*}
	We now define the learner \(h'\) such that for any finite sequence $\sigma$
	\begin{align*}
		h'(\sigma) = \begin{cases}
			h(\sigma), & \falls \content(\sigma)\subseteq C_{h(\sigma)}; \\
			s(\sigma), & \sonst.
		\end{cases}
	\end{align*}
	Note that \(\varphi_{s(\sigma)}\) and \(h'\) are total because \(h\) is a
	\(\tau(\Cind)\)-learner. Intuitively, \(h'\) has the same hypothesis as
	\(h\), if this hypothesis is consistent. Otherwise, it patches the input set
	into the hypothesis of \(h\). Thus, by construction, \(h'\) only outputs consistent
	\(C\)-indices, i.e., it is a
	\(\tau(\Cind\Cons)\)-learner.
	In particular, note that for any sequence $\sigma$ we have that
	\begin{align}
		C_{h'(\sigma)} = C_{h(\sigma)} \cup \content(\sigma). \label{Eq:ConsConstruction}
	\end{align}
	It remains to be shown that $h'$ $\delta'$-learns every language in $\La$ and
	that it obeys the restriction $\delta$ while doing so. We first show
	$\delta'$-convergence. Let $L \in \La$ and $T \in \Txt(L)$. As $h$ learns $L$,
	there exists $n_0\in\N$ such that, for all $n \geq n_0$, we have $C_{h(T[n])} = L$
	and, in the case of $\delta' = \Ex$, also $h(T[n]) = h(T[n_0])$. For $n \geq n_0$, as
	$h(T[n])$ is consistent, $h'(T[n])$ will output $h(T[n])$, proving that $h'$
	$\delta'$-learns $L$ from text $T$.

	Lastly, we prove that \(h'\) learns \(\La\) without violating the restriction
	\(\delta\). For
	\(\delta = \T\) this follows immediately. We consider the remaining
	restrictions separately. Let $L \in \La$ and let $T \in \Txt(L)$.
	\begin{enumerate}
        \itemin{1. Case:} \(\delta = \SMon\). Let $n, m \in \N$ such that $n
            \leq m$. Since $h$ is $\SMon$, we have that 
            \[
                C_{h(T[n])} \subseteq C_{h(T[m])}.  
            \] Now, by Equation~\eqref{Eq:ConsConstruction}, we
            get that $h'$ is $\SMon$ as
		    \begin{align*}
                C_{h'(T[n])} = C_{h(T[n])} \cup \content(T[n]) \subseteq
                C_{h(T[m])} \cup \content(T[m]) = C_{h'(T[m])}.
		      \end{align*}
        \itemin{2. Case:} \(\delta = \Mon\). Let $n, m \in \N$ such that $n \leq
            m$. Since $h$ is $\Mon$, we have that
		    \[
			    C_{h(T[n])} \cap \content(T) \subseteq C_{h(T[m])} \cap \content(T).
		    \]
		    Now, by Equation~\eqref{Eq:ConsConstruction}, we get that $h'$ is $\Mon$ as
		    \begin{align*}
			      C_{h'(T[n])} & \cap \content(T) = \left( C_{h(T[n])} \cup
                  \content(T[n]) \right) \cap \content(T) \\
			                   & \subseteq \left( C_{h(T[m])} \cup
                  \content(T[m]) \right) \cap \content(T) =
                  C_{h'(T[m])} \cap \content(T).
		    \end{align*}
	\end{enumerate}
	Thus, the proof is concluded.
\end{proof}

The former strategy does not work for target-cautious learners. For example, the
reason is that by simply adding missing elements, one can suddenly
overgeneralize the target language. ``Resetting'' the conjecture to solely the
information given when determining non-consistency preserves
target-cautiousness, as we show in the next result. Interestingly, this strategy
of ``resetting'' also works for weakly monotone learners as they, when
inconsistent, may propose new suggestions. We provide the next result.

\begin{lemma}
	For \(\beta \in \{\G, \Psd, \Sd\}\), \(\delta \in
	\{\WMon,{\CautTar}\}\) and $\delta' \in \{ \Ex, \Bc \}$, it holds
	\[{[\tau(\Cons)\Txt\beta\delta{\delta'}\IndR]}_\Ind =
		{[\totalCp\Txt\beta\delta{\delta'}\IndR]}_\Ind.\]
\end{lemma}

\begin{proof}
	The inclusion \({[\tau(\Cons)\Txt\beta\delta{\delta'}\IndR]}_\Ind \subseteq
    {[\totalCp\Txt\beta\delta{\delta'}\IndR]}_\Ind\) is immediate. For the
    other, it suffices to show \({[\tau(\Cind)\Txt\beta\delta{\delta_C'}]}
    \subseteq
	{[\tau(\Cind\Cons)\Txt\beta\delta{\delta_C'}]}\), due to
	\cref{Thm:ind-fam}. We use a
	similar construction as used for the case of $W$-indices, see \cite{KSS17}.
	The idea is to output solely the content of the given data if the original
	learner is not consistent. Let \(h\) be a learner and let
	$\La\subseteq
		\tau(\Cind)\Txt\beta\delta\delta'_C(h)$. We define a learner $h'$ which
	$\tau(\Cind\Cons)\Txt\beta\delta\delta'_C$-learns $\La$. For ease of notation, we use
	$h$ and $h'$ as their starred learners. We now define the
	learner \(h'\), such that for any finite sequence
	$\sigma\in\Seq$
	\begin{align*}
		h'(\sigma) = \begin{cases}
			h(\sigma),              & \falls \content(\sigma)\subseteq C_{h(\sigma)}; \\
			\ind(\content(\sigma)), & \otw.
		\end{cases}
	\end{align*}
	Note that \(h'\) is total and computable because \(h\) only outputs
	\(C\)-indices. By construction, \(h'\) only outputs consistent \(C\)-indices,
	i.e., it is a \(\tau(\Cind\Cons)\)-learner.

	It remains to be shown that $h'$ $\delta'$-learns every language in $\La$
	while obeying the restriction $\delta$. We first show $\delta'$-convergence.
	Let $L \in \La$ and $T \in \Txt(L)$. As $h$ learns $L$, there exists $n_0$
	such that, for all $n \geq n_0$, we have $C_{h(T[n])} = L$ and, in the case of
	$\delta' = \Ex$, also $h(T[n]) = h(T[n_0])$. For $n \geq n_0$, as $h(T[n])$ is consistent,
	$h'(T[n])$ will output $h(T[n])$, proving that $h'$ $\delta'$-learns $L$ on
	text $T$.

	Lastly, we prove that \(h'\) satisfies the restriction \(\delta\). We consider
	the restrictions separately. Let $L \in \La$ and let $T \in \Txt(L)$.
	\begin{enumerate}
		\itemin{1. Case:} \(\delta = \WMon\). Let $n, m \in \N$ such that $n
			      \leq m$ and $\content(T[m]) \subseteq C_{h'(T[n])}$. We show that
		      $C_{h'(T[n])} \subseteq C_{h'(T[m])}$. If $h(T[n])$ is not consistent,
		      that is, $\content(T[n]) \not\subseteq C_{h(T[n])}$, then $C_{h'(T[n])}
			      = \content(T[n])$. Thus, we have that
		      \[ C_{h'(T[n])} = \content(T[n]) \subseteq \content(T[m]) \subseteq
			      C_{h'(T[m])}. \]

		      Otherwise, $h(T[n])$ is consistent and, thus, $C_{h(T[n])} = C_{h'(T[n])}$.
		      Since, by assumption, $\content(T[m]) \subseteq C_{h'(T[n])}$ and since $h$
		      is weakly monotone, we have that $C_{h(T[n])} \subseteq C_{h(T[m])}$ and
		      also that $C_{h(T[m])}$ is consistent. Thus, in this case we get
		      \[C_{h'(T[n])} = C_{h(T[n])} \subseteq C_{h(T[m])} = C_{h'(T[m])}.\]
		\itemin{2. Case:} \(\delta = {\CautTar}\). Let $n \in \N$, then
		      $h'(T[n])$ outputs either $h(T[n])$, in which case the hypothesis is
		      target-cautious due by assumption, or it outputs $\ind(\content(T[n]))$.
		      As \(\content(T[n])\subseteq \content(T)\), this hypothesis is also
		      target-cautious.\qedhere
	\end{enumerate}
\end{proof}

Although (semantically) conservative learners may also change their mind upon
inconsistency, the same strategy does not work. The problem is that one may make
them consistent too early and, thus, prevent later mind changes from happening.
An interesting strategy solves the problem. One mimics the (possibly)
inconsistent learner on information without repetition. Learning is preserved
this way, as when inferring infinite target languages there will always be new
information to correct an incorrect conjecture. On the other hand, finite target
languages serve no problem either as, given all information without repetition,
either the learner was correct anyway or making it consistent is a correct
guess.

\begin{lemma}
	For \(\beta \in \{\G, \Psd, \Sd\}\), \(\delta \in \{\SemConv, \Conv\}\)
	and $\delta' \in \{ \Ex, \Bc \}$, we have
	\[{[\tau(\Cons)\Txt\beta\delta{\delta'}\IndR]}_\Ind =
		{[\totalCp\Txt\beta\delta{\delta'}\IndR]}_\Ind.\]
\end{lemma}

\begin{proof}
	The inclusion \({[\tau(\Cons)\Txt\beta\delta{\delta'}\IndR]}_\Ind \subseteq
	{[\totalCp\Txt\beta\delta{\delta'}\IndR]}_\Ind\) is immediate. For the other, we use a
	similar construction as when learning $W$-indices as presented in
	\cite{KSS17}. By \cref{Thm:ind-fam}, it suffices to show
	\({[\tau(\Cind)\Txt\beta\delta{\delta'}_C]} \subseteq
	{[\tau(\Cind\Cons)\Txt\beta\delta{\delta'}_C]}\). Let \(h\) be a learner and
	let $\La\subseteq
		\tau(\Cind)\Txt\beta\delta\delta'_C(h)$. We define a learner $h'$ which
	$\tau(\Cind\Cons)\Txt\beta\delta\delta'_C$-learns $\La$. For ease of notation, we use
	$h$ and $h'$ as their starred learners. Given a sequence $\sigma$, we write
	$\tilde{\sigma}$ for the sequence without repetitions or pause symbols.
	Analogously, $\Psd$-learners given $(\content(\sigma),|\sigma|)$ consider
	$(\content(\tilde{\sigma}),|\tilde{\sigma}|)$ instead. Notably, $\Sd$-learners
	receive the same information as $\content(\sigma) = \content(\tilde{\sigma})$.
	Now, we define $h'$ such that, for all finite sequences $\sigma\in\Seq$,
	\begin{align*}
		h'(\sigma) =
		\begin{cases}
            h(\tilde{\sigma}),              & \falls
                \content(\tilde{\sigma})\subseteq C_{h(\tilde{\sigma})}; \\
			\ind(\content(\tilde{\sigma})), & \sonst.
		\end{cases}
	\end{align*}
	Note that, by construction, \(h'\) is a \(\tau(\Cind\Cons)\)-learner.
	The intuition for the learner $h'$ is then to mimic $h$ on information
	without repetition. This is important to ensure (semantic) conservativeness.
	Given $\sigma$, it either outputs the same hypothesis as
	$h(\tilde{\sigma})$, if this is
	consistent, or it outputs solely a \(C\)-index for the content of the input.

    Next, we show that $h'$ $\delta'$-learns $\La$. Let $L \in \La$ and $T \in
    \Txt(L)$. We distinguish whether $L$ is finite or not.
	\begin{enumerate}
		\itemin{1. Case:} $L$ is finite. Then, there exists a minimal $n_0\in\N$ such
		      that $\content(T[n_0]) = L$. Then, by definition, for all $n \geq n_0$,
		      we have that $h'(T[n_0]) = h'(T[n])$ as no new element will be
		      witnessed. Now, if $h(T[n_0])$ is consistent, then, because \(h\) is
		      (semantically) conservative and thus target-cautious,  we have
		      $C_{h(T[n_0])} = L$. Otherwise, $h'$
		      outputs $\ind(\content(T[n_0]))$. In both cases, $h'(T[n_0])$ is a
		      correct hypothesis.
		\itemin{2. Case:} $L$ is infinite. Note that the transition to text $T$ from
		      the corresponding text $T' \in \Txt(L)$ which does not contain any
		      duplicates or pause-symbols can be done using an unbounded,
		      non-decreasing function $r\colon \N \to \N$, that is, $T = T' \circ r$.
		      As $\delta'$ is delayable, it suffices to show the convergence on text
		      $T'$. As $h$ also converges on $T'$, there exists some $n_0$ such that,
		      for all $n \geq n_0$, we have $C_{h(T'[n])} = L$ and, if $\delta' =
			      \Ex$, also $h(T'[n_0]) = h(T'[n])$. In particular, for all $n \geq
			      n_0$, $h(T'[n])$ is consistent and, thus, $h'(T'[n]) =
			      h(\widetilde{T'[n]}) = h(T'[n])$. Thus, $h'$ $\delta'$-learns $L$ from
		      text $T'$ since $h$ does and, as $\delta'$ is delayable, $h'$ also
		      learns $L$ from text~$T$.
	\end{enumerate}
    It remains to be shown that $h'$ obeys $\delta$. The basic idea is that $h'$
    may only make a mind change if it sees a new element which is not consistent
    with the current hypothesis. Formally, let $L \in \La$ and $T \in \Txt(L)$.
    Furthermore, let $n,m \in \N$, with $n < m$, be such that $\content(T[m])
    \subseteq C_{h'(T[n])}$. We distinguish between the two cases for $\delta$.
	\begin{enumerate}
		\itemin{1. Case:} $\delta = \SemConv$. We show that $C_{h'(T[n])} =
			      C_{h'(T[m])}$. In the case of $\content(T[n]) = \content(T[m])$, this
		      follows by definition. Otherwise, there exists an element in
		      $\content(T[m])$ which is not in $\content(T[n])$. Thus, in order for
		      $C_{h'(T[n])}$ to enumerate $\content(T[m])$, that is, $\content(T[m])
			      \subseteq C_{h'(T[n])}$, it must hold that $\content(T[m]) \subseteq
			      C_{h(\widetilde{T[n]})}$. Then, since $h$ is semantically conservative,
		      we have $C_{h(\widetilde{T[n]})} = C_{h(\widetilde{T[m]})}$. In
		      particular, $h(\widetilde{T[m]})$ is consistent, meaning that
		      $C_{h'(T[m])} = C_{h(\widetilde{T[m]})}$. Altogether, we get
		      \[
                  C_{h'(T[n])} = C_{h(\widetilde{T[n]})} =
                  C_{h(\widetilde{T[m]})} = C_{h'(T[m])}.
		      \]
		\itemin{2. Case:} $\delta = \Conv$. This case follows an analogous
		      proof-idea, the main difference being that semantic equalities need to
		      be replaced with syntactic ones. We show that $h'(T[n]) = h'(T[m])$. In
		      the case of $\content(T[n]) = \content(T[m])$, this follows by
		      definition. Otherwise, there exists an element in $\content(T[m])$ which
		      is not in $\content(T[n])$. Thus, in order for $C_{h'(T[n])}$ to
		      enumerate $\content(T[m])$, that is, $\content(T[m]) \subseteq
			      C_{h'(T[n])}$, it must hold that $\content(T[m]) \subseteq
			      C_{h(\widetilde{T[n]})}$. Then, since $h$ is (syntactically)
		      conservative, we have $h(\widetilde{T[n]}) = h(\widetilde{T[m]})$. In
		      particular, $h(\widetilde{T[m]})$ is consistent, meaning that $h'(T[m])
			      = h(\widetilde{T[m]})$. Altogether, we get
		      \[h'(T[n]) = h(\widetilde{T[n]}) = h(\widetilde{T[m]}) = h'(T[m]), \]
		      which concludes the proof.\qedhere
	\end{enumerate}
\end{proof}

This concludes the proof of Theorem~\ref{th:taucons}. We note that none of these
strategies work for, say, non-U-shaped learners as one may, by patching in
missing elements into the hypotheses or ``resetting'' them, accidentally produce
a hypothesis for the target language and later forget it again. Thus, it remains
an open question whether learning under such restrictions even allows for
consistent learning. We pose the following question.

\begin{problem}
Does $\delta \in \{ \NU, \SNU, \Dec, \SDec\}$ allow for consistent learning,
that is, for $\beta \in \{ \G, \Psd, \Sd\}$ and $\delta' \in \{ \Ex, \Bc\}$,
does it hold that
\begin{align*}
	{[\tau(\Cons)\Txt\beta\delta{\delta'}\IndR]}_\Ind =
	{[\totalCp\Txt\beta\delta{\delta'}\IndR]}_\Ind?
\end{align*}
\end{problem}

\section{Delayable Map for Learning Indexed Families}\label{Sec:DelMap}

In this section we compare the power of (possibly partial) learners following
various delayable learning restrictions to each other. First, we gather known
results from literature. It is a well-known fact that learners need time in
order to obtain full learning power, that is, set-driven learners lack learning
power. The following theorem holds.

\begin{theorem}[\cite{LZ96}] \label{thm:g-sd-ex}
	We have that
	\(
	{[\Txt\G\Ex\IndR]}_\Ind \setminus {[\Txt\Sd\Ex\IndR]}_\Ind\neq\emptyset.
	\)
\end{theorem}

Furthermore, in the literature monotonic learners have been investigated
thoroughly. Interestingly, a chain of inclusions is obtained. The following
theorem holds.

\begin{theorem}[\cite{LZ96}]
	We have that
	\begin{align*}
		{[\Txt\Sd\SMon\Ex\IndR]}_\Ind & = {[\Txt\G\SMon\Ex\IndR]}_\Ind \subsetneq
		{[\Txt\Sd\Mon\Ex\IndR]}_\Ind \subsetneq                                   \\ 
                                      &\subsetneq
		{[\Txt\G\Mon\Ex\IndR]}_\Ind \subsetneq {[\Txt\Sd\WMon\Ex\IndR]}_\Ind =    \\
		                              & = {[\Txt\G\WMon\Ex\IndR]}_\Ind.
	\end{align*}
\end{theorem}

We remark that weak monotonicity as well as conservativeness is no restriction
to set-driven learners \cite{LZ96}. We expand this result by showing that
set-driven learners may be assumed to be even witness-based. This way, we also
capture the remaining restrictions, such as (target-) cautiousness and (strong)
decisiveness. To obtain the desired result, we first show that target-cautious
and witness-based Gold-style learners acquire the same learning power. The idea
is that, as target-cautious learners never overgeneralize the target language,
there always remain elements as witnesses to justify a mind change if the
current hypothesis is wrong. We obtain the following result.

\begin{theorem}
	We have that ${[\Txt\G\Wb\Ex\IndR]}_\Ind = {[\Txt\G\CautTar\Ex\IndR]}_\Ind$.
\end{theorem}

\begin{proof}
	The inclusion \( {[\Txt\G\Wb\Ex\IndR]}_\Ind \subseteq
	{[\Txt\G\CautTar\Ex\IndR]}_\Ind \) is straightforward. For the other, by
	assuming that $\G$-learner are total, see \cref{thm:g-wlog-total}, and by
	\cref{Thm:ind-fam}, it suffices to show
	\[
		{[\tau(\Cind)\Txt\CautTar\G\Ex_C]} \subseteq
		{[\tau(\Cind)\Txt\G\Wb\Ex_C]}.
	\]
	Let \(h\) be a learner with $\La\subseteq \tau(\Cind)\Txt\G\CautTar\Ex_C(h)$.
	Using \cref{th:taucons}, we can assume that \(h\) is consistent, i.e.,
	$\La\subseteq\tau(\Cind\Cons)\Txt\CautTar\G\Ex_C(h)$. We now prove that the following
	learner \(h'\) is a \(\tau(\Cind)\Txt\G\Wb\Ex_C\)-learner for \(\La\). Let
	\(h'(\varepsilon)=h(\varepsilon)\) and, for all finite
	\(\sigma\in\Seq\) and \(x\in\N\), let
	\begin{align*}
		h'(\sigma\concat x) \coloneqq \begin{cases} 
            h'(\sigma),         & \cIf x \in C_{h'(\sigma)}; \\
			h(\sigma\concat x), & \otw.
		\end{cases}
	\end{align*}
    Intuitively, \(h'\) only updates its hypothesis if the latest datum may be
    used as a witness for a mind change. As $h$ is also consistent, we
    immediately have that $h'$ is witness-based.
	Furthermore, note that \(h'\) outputs a \(C\)-index on every input and thus is a
	\(\tau(\Cind)\)-learner.

	It remains to be shown that \(h'\) learns \(\La\).
	To that end, let $L \in \mathcal{L}$ and $T\in\Txt(L)$. Since $h$ correctly
    learns $L$ there exists $n_0 \in \N$ such that, for all \(n\geq n_0\), we
    have $h(T[n]) = h(T[n_0])$ and
	$\varphi_{h(T[n])}=\chi_{L}$. We distinguish the following cases.
	\begin{enumerate}
        \itemin{1. Case:} $h'(T[n_0])$ is a $C$-index for $L$. In this case, $L
            \setminus C_{h'(T[n_0])} = \emptyset$, and thus $h'$ cannot change
            its mind again. Thus, it converges correctly.
        \itemin{2. Case:} $h'(T[n_0])$ is no $C$-index for $L$. As $h$ is
            target-cautious and $h'$ mimics $h$, it cannot hold $L \subsetneq
            C_{h'(T[n_0])}$. Thus, there exists $x \in L$ with \(x\notin
            C_{h'(T[n_0])}\). By definition of $h'$ and by consistency of $h$,
            we have $x \notin \content(T[n_0])$. Let $n_1$ be such that $x \in
            \content(T[n_1])$. Then, by construction, for all \(n \geq n_1\), we
            have that \(h'(T[n])=h(T[n_1])\), which is a \(C\)-index for \(L\).
            \qedhere
	\end{enumerate}
\end{proof}

This equality also includes weakly monotone learners. Thus, we already have that
these are as powerful as set-driven learners. However, we go one step further
and show that these learners may even be assumed total. We make use of
Theorems~\ref{thm:ind-fam} and~\ref{thm:g-wlog-total}. The idea is to mimic the
Gold-style learner on the minimal, consistent hypothesis. This way,
target-cautiousness is preserved as well as learning power. The latter works out
as no guess overgeneralizes the target language and, thus, checking for
consistency is a valid strategy. The following theorem holds.

\begin{theorem}
	We have that 
    ${[\totalCp\Txt\Sd\CautTar\Ex\IndR]}_\Ind = {[\Txt\G\CautTar\Ex\IndR]}_\Ind$.
\end{theorem}

\begin{proof}
	The inclusion \({[\totalCp\Txt\Sd\CautTar\Ex\IndR]}_\Ind
	\subseteq {[\Txt\G\CautTar\Ex\IndR]}_\Ind \) is immediate. For the other one,
    as $\G$-learner may be assumed total, see \cref{thm:g-wlog-total}, and by
    Theorem~\ref{Thm:ind-fam}, it suffices to show
	\[
		{[\tau(\Cind)\Txt\G\CautTar\Ex_C]} \subseteq
		{[\tau(\Cind)\Txt\Sd\CautTar\Ex_C]}.
	\]
	Let $h$ be a learner and let $\La\subseteq
		\tau(\Cind)\Txt\G\CautTar\Ex_C(h)$. By Theorem~\ref{th:taucons}, we may assume
	$h$ to be consistent. For a finite set $D$, for $k \leq \abs{D}$, let $\sigma_D[k]$
	be the canonical sequence of $D$ of length $k$, that is, the sequence of the
	first $k$ elements in $D$ in strictly ascending order, and define the
	$\Sd$-learner $h'$ as
	\[
		h'(D) = h(\sigma_D[\min\{k' \leq |D| \mid D \subseteq C_{h(\sigma_D[k'])} \}]).
	\]
	Mimicking learner $h$, the newly defined learner $h'$ is target-cautious
	whenever $h$ is and it always outputs $C$-indices. It remains to be shown that
	$h'$ $\Sd\Ex_C$-learns $\La$. Let, to that end, $L \in \La$. We distinguish the
	following cases.
	\begin{enumerate}
        \itemin{1. Case:} $L$ is finite. Let $k_0\in\N$ be the minimal $k' \leq
            |L|$ such that $L \subseteq C_{h(\sigma_L[k'])}$. By consistency of
            $h$, such $k'$ exists. Then, we have by $h$ being target-cautious
            that $\neg (L \subsetneq C_{h(\sigma_L[k_0])})$. Altogether, we have
		      \[
			      C_{h'(L)} = C_{h(\sigma_L[k_0])} = L.
		    \]
		\itemin{2. Case:} $L$ is infinite. Then, consider the canonical text $T$ of
            $L$. As $h$ learns $L$, there exists a minimal $n_0\in\N$ such that
            $C_{h(T[n_0])} = L$.  By target-cautiousness of $h$ and minimal
            choice of $n_0$, there exists $n_1 \geq n_0$ such that for all $n <
            n_0$ we have
		      \[
			      \content(T[n_1]) \setminus C_{h(T[n])}.
		      \]
              Then, for all $D$ with $\content(T[n_1]) \subseteq D \subseteq L$,
              we have $h'(D) = h(T[n_0])$ as desired. \qedhere
	\end{enumerate}
\end{proof}

Again with Theorem~\ref{thm:ind-fam}, we obtain that set-driven learners may be
assumed total and witness-based. The idea resembles the approach for partially
set-driven learners of arbitrary classes of languages \cite{KS16}. To obtain
this, we assume the information coming in a certain order and then mimic the
learner on the least input where no mind change is witnessed. Then, while
enumerating, we check whether any later datum causes a mind change. If so, we
stop the enumeration. Especially here, Theorem~\ref{thm:ind-fam} comes in handy
as we do not need to fix the hypothesis space beforehand, but rather build it up
on the fly. The following result holds.

\begin{theorem}%
	\label{Thm:SdWb}
	We have that $[\totalCp\Txt\Sd\Wb\Ex]_\Ind = [\totalCp\Txt\Sd\CautTar\Ex]_\Ind$.
\end{theorem}

\begin{proof}
    The direction $[\totalCp\Txt\Sd\Wb\Ex]_\Ind \subseteq
    [\totalCp\Txt\Sd\CautTar\Ex]_\Ind$ follows immediately. For the other, by
    Theorem~\ref{Thm:ind-fam}, it suffices to show that
	\[
		[\tau(\Cind)\Txt\Sd\CautTar\Ex_C] \subseteq [\tau(\Cind)\Txt\Sd\Wb\Ex_C].
	\]
    Let $h$ $\tau(\Cind)\Txt\Sd\CautTar\Ex_C$-learn $\La$. We define the desired
    witness-based learner $h'$. Given a finite set $D$ and $k \leq |D|$, let
    $D[k]$ be the set of the first $k$ elements (in ascending order) in $D$, and
    define
	\[
        k_D = \min\{ k \leq |D| \mid \forall D', D[k] \subseteq D' \subseteq D
        \colon h(D') = h(D) \}.
	\]
    That is, $D[k_D]$ contains the minimal amount of elements of $D$ in
    ascending order where no mind change is witnessed. Furthermore, for any
    finite set $D$, define
	\[
		\varphi_{s(D)}(x) = \begin{cases}
			1, & \falls x \in D;\\
            0, & \sonstfalls \varphi_{h(D)}(x) = 0;\\
            1, & \sonstfalls \forall D', D \subseteq D' \subseteq D \cup
                C_{h(D)}^x \colon h(D) = h'(D); \\
			0, & \sonst.
		\end{cases}
	\]
    Note that for any locking set $D$ of some language $L$, we have $C_{s(D)} =
    L$. Then, for any finite set $D$, we define the learner
	\[
		h'(D) = \begin{cases}
            \ind(D[k_D]), & \falls \exists x < \max(D[k_D]), x \notin D[k_D]
                \colon \varphi_{h(D[k_D])}(x) = 1; \\
			s(D[k_D]),    & \sonst.
		\end{cases}
	\]
    Intuitively, the learner first searches the minimal amount of elements where
    no mind change is witnessed. Then, given $D[k_D]$, if the learner on input
    $D[k_D]$ witnesses an element to be (possibly) out of order, it outputs
    $\ind(D[k_D])$. This way, we keep this element as a witness for a possible,
    later mind change. Otherwise, the learner outputs $s(D[k_D])$ which conducts
    a forward search and enumerates all elements where no mind change is
    witnessed.

    Formally, we first show that $h'$ learns $\La$ correctly. Let therefore $L
    \in \La$. We distinguish the following cases.
	\begin{enumerate}
		\itemin{1. Case:} $L$ is finite. Here, $h(L)$ is a correct conjecture and,
		      thus, $h(L[k_L])$ as well. Note that, in particular, $L[k_L]$ is a
		      locking set for $L$. As there exists no $x < \max(L[k_L])$ with
              \( \varphi_{h(L[L_k])}(x)=1\), we have that $h'(L[k_L]) =
              s(L[k_L])$. Since $L[k_L]$ is a
		      locking set, $s(L[k_L])$ is a correct conjecture.
		\itemin{2. Case:} $L$ is infinite. Let $T_c$ be the canonical text for $L$ and let
		      $n_0\in\N$ be minimal such that $D_0 \coloneqq \content(T_c[n_0])$ is a
		      locking set for $h$ on $L$. As, for all $n < n_0$, $\content(T_c[n])$ is
		      no locking set, there exists some $x_n \in L$ where this is witnessed. Let
              $x_{\max} = \max\{ x_n \mid n < n_0 \}$ and let $n_1 \geq n_0$
              such that, for $D_1 \coloneqq \content(T_c[n_1])$, we have
              $x_{\max} \in D_1$. Then, for any $D$ with $D_1 \subseteq D
              \subseteq L$, we have that $h'(D) = h'(D_1)$. As $D_1$ is a
              locking set, we have that $h'(D_1) = s(D_1)$ is a correct
              hypothesis.
	\end{enumerate}
    Lastly, we show that $h'$ is witness-based. Let, to that end, $D_1 \subseteq
    D_2 \subseteq D_3 \subseteq L$ such that $h'(D_1) \neq h'(D_2)$. We show
    that
	\(
	(C_{h'(D_3)} \cap D_3) \setminus C_{h'(D_1)} \neq \emptyset.
	\)
    For $i \in \{ 1,2,3\}$, let $k_i \coloneqq k_{D_i}$ and let $D'_i \coloneqq
    D[k_i]$. Then, as $h'(D_i) = h'(D'_i)$, it suffices to show
	\[
		(C_{h'(D'_3)} \cap D_3) \setminus C_{h'(D'_1)} \neq \emptyset.
	\]
	We distinguish the following cases.
	\begin{enumerate}
        \itemin{1. Case:} $D_1' = D_3'$. In particular, $D_1' = D_2' = D_3'$.
            Then, $h'(D_1) = h'(D_2)$, a contradiction to the initial
            assumption.
		\itemin{2. Case:} $D_3' \setminus D_1' \neq \emptyset$. Let $x$ be a maximal such
            element. Either, $x \notin C_{h(D_1')}$ and, thus by
            Condition~\eqref{Bed:Ccons}, it will not be considered when
            enumerating $C_{h'(D_1')}$. Otherwise, $x \notin C_{h'(D_1')}$ as
            it either is smaller than
		    $\max(D_1')$ or it will not be enumerated by $s(D_1')$ as it witnesses a
		    mind change.
        \itemin{3. Case:} $D_1' \setminus D_3' \neq \emptyset$. If $D_3'
            \subseteq D_1'$, then, as $D_1 \subseteq D_3$, the minimality of
            $k_1$ is violated. Thus, it also holds that $D_3' \setminus D_1'
            \neq \emptyset$, and we proceed just as in the previous case.
            \qedhere
	\end{enumerate}
\end{proof}

This closes the study of set-driven learners following delayable learning
restrictions. It remains to be shown that Gold-style learners may be assumed
strongly decisive. We do so in two steps. First, we show that unrestricted
learners may be assumed strongly non-U-shaped in general. The idea is to search
for locking sequences. If we witness that the current sequence is not locking,
we \emph{poison} the produced hypothesis \cite{CK16}. We can do so, as indexed
families provide a decision procedure to check whether $x \in L_i$ or not. When
poisoning, we simply output a hypothesis contradicting all of the given
languages. Note that, by Theorem~\ref{thm:ind-fam}, we may construct poisoned
hypotheses on the fly.The following theorem holds.

\begin{theorem}
	We have that ${[\Txt\G\SNU\Ex\IndR]}_\Ind = {[\Txt\G\Ex\IndR]}_\Ind$.
\end{theorem}

\begin{proof}
	The inclusion ${[\Txt\G\SNU\Ex\IndR]}_\Ind \subseteq {[\Txt\G\Ex\IndR]}_\Ind$
	is immediate. For the other direction, note that $\G$-learners may be assumed
    total by \cref{thm:g-wlog-total}. Thus, by Theorem~\ref{Thm:ind-fam}, it
    suffices to show~that
	\[
		{[\tau(\Cind)\Txt\G\Ex_C]} \subseteq {[\tau(\Cind)\Txt\G\SNU\Ex_C]}.
	\]
	Let $h$ be a $\tau(\Cind)\Txt\G\Ex_C$-learner and let
	$\La\subseteq
		\tau(\Cind)\Txt\G\Ex_C(h)$. We provide a $\tau(\Cind)\Txt\G\SNU\Ex_C$-learner
	for $\La = {(L_i)}_{i \in \N}$. The idea is the following. Since $\La$ is
	indexed, there exists a procedure to decide whether $x \in L_i$ or not. Given
    any input, we check whether it serves as a locking sequence. Note that
    $\Txt\G\Ex$-learners may be assumed strongly locking \cite{Fulk90}. While it
    does so, we mimic the learner on this input. Once we figure it not being a
    locking sequence, we start \emph{poisoning} this guess by contradicting it
    to each possible language $L_i$. Thus, the resulting learner will output the
    correct language once it finds a locking sequence thereof.

	Formally, we first define the auxiliary predicate which, given a sequence
    $\sigma$ and an element $x\in\N$, tells us whether $\sigma$ is a candidate
    for a locking sequence up until the element $x$, that is,
	\[
		Q(\sigma, x) = \begin{cases}
			1, & \falls \exists \sigma' \in {\left( C_{h(\sigma)}^x \right)}^{\leq x}_\#,
			\sigma \subseteq \sigma' \exists y \leq x \colon \varphi_{h(\sigma)}(y) \neq
			\varphi_{h(\sigma')}(y);                                                      \\
			0, & \sonst.
		\end{cases}
	\]
    We use $C_{h(\sigma)}^x$ to denote all elements in $C_{h(\sigma)}$ up until
    $x$, that is, $C_{h(\sigma)}^x = \{ x' \leq x \mid \varphi_{h(\sigma)}(x') =
    1 \}$. Next, the S-m-n Theorem provides us with an auxiliary function which
    poisons conjectures on non-locking sequences. There exists $s \in \totalCp$
    such that for all $x\in\N$ and \(\sigma\in\Seq\)
	\[
		\varphi_{s(\sigma)}(x) = \begin{cases}
			\varphi_{h(\sigma)}(x), & \falls Q(\sigma, x) = 0;\\
            0,                      & \sonstfalls x \in L_{x - \min\{ y \in \N
                \mid Q(\sigma, y) = 1 \}}; \\
			1,                      & \sonst.
		\end{cases}
	\]
	Note that in the second case $\{ y \in \N \mid Q(\sigma, y) = 1 \}$ is non-empty
	(as the first case does not hold) and that its elements are bound by $x$.
	Lastly, we need the following auxiliary function which finds the minimal
	sequence on which $h$ agrees with the current hypothesis up until some point.
	For any sequence $\sigma$, define
	\[
        M(\sigma)=\{\sigma' \subseteq \sigma \mid \forall \sigma'' \in
            {\content(\sigma)}^{\leq |\sigma|}_\# \ \forall x \leq
            |\sigma|\colon\varphi_{h(\sigma)}(x) = \varphi_{h({\sigma'}\concat\sigma'')}
            (x)\}.
	\]
	Finally, we define the learner $h'$ as, for any sequence $\sigma$,
	\[
		h'(\sigma) = s(\min(M(\sigma))).
	\]
	Now, let $L \in \La$ and let $T \in \Txt(L)$. We show that $h'$
	converges to a correct hypothesis and, afterwards, show this learning to be
	strongly non-U-shaped. As $h$ is strongly locking, there exists a minimal
    $n_0\in\N$ such that $T[n_0]$ is a locking sequence for $h$ on $L$. In
    particular, there exists $n_1 \geq n_0$ such that, for all $n < n_0$, $T[n]
    \notin M(T[n_1])$, that is, we witness all sequences prior to $T[n_0]$ not
    to be locking. Then, for all $n \geq n_1$, we have that $\min(M(T[n])) =
    T[n_0]$ and, thus, $h'(T[n]) = s(T[n_0])$. Furthermore, for any
	$x \in\N$, we have that $Q(T[n_0], x) = 0$ as all the sequences output the same
	hypothesis. Thus, $\varphi_{s(T[n_0])} = \varphi_{h(T[n_0])}$, meaning that
	$s(T[n_0])$ is a $C$-index for $L$.

    We now show that this learning is strongly non-U-shaped. First, we show, for
    all \(n<n_0\), that \(s(T[n])\) is no \(C\)-index for \(L\). By
	minimality of \(n_0\), \(T[n]\) is no locking sequence for \(h\) on \(L\).
	Now, if $h(T[n])$ is no
	$C$-index of $L$, neither will $s(T[n])$ be, as it either outputs the same
	as $h(T[n])$ or eventually contradicts all languages in $\La$. If,
	otherwise, $h(T[n])$ is a $C$-index of $L$, there exists some point $x\in\N$
	witnessing $T[n]$ not to be a locking sequence. Then, $s(T[n])$ starts
	contradicting all languages in $\La$. Thus, \(s(T[n])\) and also  \(h'(T[n])\)
	is no \(C\)-index for \(L\).
\end{proof}

Building on this result, we go one step further and show the learners to be even
strongly decisive. The strategy the newly found learner employs is to wait with
changing its hypothesis until it witnesses a mind change. And, when doing so, it
first checks whether this mind change produces a new hypothesis which is
different from all previous ones. The following theorem holds.

\begin{theorem}
	We have that ${[\Txt\G\SDec\Ex\IndR]}_\Ind = {[\Txt\G\SNU\Ex\IndR]}_\Ind$.
\end{theorem}

\begin{proof}
	The inclusion ${[\Txt\G\SDec\Ex\IndR]}_\Ind \subseteq
		{[\Txt\G\SNU\Ex\IndR]}_\Ind$ follows immediately. For the other, it suffices,
	by the observation that $\G$-learners may be assumed total
	(\cref{thm:g-wlog-total}) and by Theorem~\ref{Thm:ind-fam}, to show that
	\[
		{[\tau(\Cind)\Txt\G\SNU\Ex_C]} \subseteq
		{[\tau(\Cind)\Txt\G\SDec\Ex_C]}.
	\]
	To that end, let $h$ be a learner and let $\La\subseteq
		\tau(\Cind)\Txt\G\SNU\Ex_C(h)$. We define an equivalently powerful
	$\tau(\Cind)\Txt\G\SDec\Ex_C$-learner $h'$ as follows. Let $h'(\varepsilon) =
		h(\varepsilon)$ and, for any finite sequence $\sigma \neq \varepsilon$, let
    $\sigma' \subsetneq \sigma$ be the minimal sequence on which $h'(\sigma') =
    h'(\sigma^-)$, that is, the sequence on which $h'$ based its previous
    output. Then, define
	\[
		h'(\sigma) = \begin{cases}
            h(\sigma'), & \falls \forall \sigma'', \sigma' \subseteq \sigma''
                \subseteq \sigma \colon h(\sigma') = h(\sigma'');\\
            h(\sigma),  & \sonstfalls \forall \sigma'' \subseteq \sigma' \ 
                \exists x \leq |\sigma| \colon \varphi_{h'(\sigma'')}(x) \neq
                \varphi_{h(\sigma)}(x); \\
			h(\sigma'), & \sonst.
		\end{cases}
	\]
    As $h'$ mimics $h$, $h'$ always outputs $C$-indices and, hence, is
    $\tau(\Cind)$. The intuition is to only update the hypothesis if the current
    hypothesis cannot be based on a locking sequence and if all previous ones
    are witnessed to be semantically different. As $h$ is $\SNU$, $h'$ may never
    abandon a correct guess and all hypotheses before that are incorrect. Thus,
    $h'$ preserves the learning power.

    Formally, we first show that $h'$ is, indeed, $\SDec$. We do so by showing
    that whenever $h'$ makes a mind change, this new hypothesis is certainly
    semantically different from all previous ones and, thus, also syntactically
    different. Let $L \in \La$ and let $\sigma \in L_\#^*$ such that
    $h'(\sigma^-)
		\neq h'(\sigma)$, that is, $h$ made a mind change. Note that $h'(\sigma) =
		h(\sigma)$. Furthermore, let $h'$ base its prior hypothesis on 
    $\sigma' \subseteq \sigma^-$, that is, $\sigma' \subseteq \sigma^-$ is the
    minimal sequence on which, for all $\sigma''$ with $\sigma' \subseteq
    \sigma'' \subseteq \sigma^-$, we have $h'(\sigma'') = h'(\sigma^-)$. The
    only case where $h'$ makes a mind change is, if for all $\sigma'' \subseteq
    \sigma'$ there exists $x \leq |\sigma|$ such that
	\[
		\varphi_{h'(\sigma'')}(x) \neq \varphi_{h(\sigma)}(x).
	\]
    As $h'(\sigma) = h(\sigma)$ and, therefore, $\varphi_{h(\sigma)} =
    \varphi_{h'(\sigma)}$, we have, for all $\tilde{\sigma} \subseteq \sigma'$,
	\[
		C_{h'(\tilde{\sigma})} \neq C_{h'(\sigma)}.
	\]
    As there are no further mind changes until $\sigma^-$, this holds for all
    $\tilde{\sigma} \subsetneq \sigma$. Thus, $h'$ is $\SDec$.

	To show that $h$ converges correctly, let $L \in \La$ and let $T \in \Txt(L)$.
	Then there exists a (minimal) $n_0$ such that, for all $n \geq n_0$, $h(T[n_0])
		= h(T[n])$ and $h(T[n])$ is a $C$-index for $L$.
	We distinguish the following cases.
	\begin{enumerate}
        \itemin{1. Case:} $h'(T[n_0]) = h(T[n_0])$. In this case, $h'(T[n_0])$
            is a correct $C$-index and, as $h$ never changes its mind again,
            neither does $h'$.
        \itemin{2. Case:} $h'(T[n_0]) \neq h(T[n_0])$. Let $n_1 < n_0$ be such
            that $h'(T[n_0]) = h(T[n_1])$. In particular, $h(T[n_1]) \neq
            h(T[n_0])$. Thus, the first case of the definition of $h'$ cannot
            hold. By $h$ being $\SNU$ and by the minimal choice of $n_0$, there
            exists some minimal $n_2 \geq n_0$ such that $h'$ witnesses all
            hypotheses prior to (and including) $h(T[n_1])$ to differ from
            $h(T[n_2])$. Then, by the second case of the definition, it will
            output $h(T[n_2])$ never to change its mind again. \qedhere
	\end{enumerate}
\end{proof}

Altogether, we obtain the full map as depicted in
Figure~\ref{fig:tauCindMapGSd}. It remains open to include partially set-driven
learners into this picture. Due to known results from literature and the results
we obtained, in particular, it remains to be shown whether Gold-style learners
may be assumed strongly decisive and partially set-driven at the same time. We
pose the following open question.

\begin{problem}
May Gold-style strongly decisive learners be assumed partially set-driven so?
\end{problem}

\section{Comparing Convergence Criteria when Learning Indexable
    Classes}\label{Sec:tCindConvergence}
\label{sec:convergence_tau_cind}

In this section we compare total learners under various memory constraints which
converge syntactically to such that converge semantically. The gathered results
we depict in Figure~\ref{fig:tauCindMemory}. Usually, semantically converging
learners are more powerful than their syntactic counterpart, for example, when
learning arbitrary classes of languages \cite{Fulk90}. However, when learning
indexed families class-comprisingly different results are obtained. It is known
that explanatory $\G$-learner and behaviorally correct ones are equally powerful
\cite{LZZ08}. As $\G$-learner may be assumed total, see
Theorem~\ref{thm:g-wlog-total}, we obtain the following result.

\begin{theorem}[\cite{LZZ08}]
	We have that ${[\totalCp\Txt\G\Ex\IndR]}_\Ind = {[\totalCp\Txt\G\Bc\IndR]}_\Ind$.
\end{theorem}

Furthermore, it is known that Gold-style learners do not rely on the order of
the presented elements but rather on the time given. The latter result we
already discussed in Theorem~\ref{thm:g-sd-ex} and note that it also holds true
for total learners, the former is known to hold true for (possibly) partial
learners. In order to obtain this result, one searches for the minimal candidate
for a locking sequence and mimics the learner on it. As $\G$-learner may be
assumed total, see Theorem~\ref{thm:g-wlog-total}, one obtains a total
$\Psd$-learner this way. Thus, the following theorem holds.

\begin{theorem}%
	\label{thm:g-psd}
	We have that \({[\totalCp\Txt\Psd\Ex]}\IndF={[\Txt\G\Ex]}\IndF\).
\end{theorem}

\begin{proof}
	The inclusion \({[\totalCp\Txt\Psd\Ex]}\IndF \subseteq {[\Txt\G\Ex]}\IndF\)
    is straightforward. For the other, we follow the proof in~\cite{Fulk90}. Let
    \(h\) \(\Txt\G\Ex\)-learn $\La$ with respect to a hypothesis space $\Ha$.
    Without losing generality, see \cref{thm:g-wlog-total}, let \(h\) be total.
    We define a \(\Psd\)-learner \(h'\) using an auxiliary function
    \(M\in\totalCp\) as, for all finite sets \(D\subseteq\N\) and \(t\in\N\),
	\begin{align*}
		M(D,t)  & = \set{\sigma\in D^{\leq t}_\#}{\forall\tau\in D^{\leq
		    t}_\#\colon h(\sigma)=h(\sigma\tau) };\\
		h'(D,t) & = \begin{cases}
			h(\min(M(D,t))), & \cIf M(D,t)\neq\emptyset; \\
		    h(\varepsilon),  & \otw.
		\end{cases}
	\end{align*}
	Intuitively, \(h'\) mimics \(h\) on minimal potential locking sequences.
	Note that \(h'\) is total as \(h\) is so. To show that \(h\) learns \(\La\),
	let \(L\in\La\) and \(T\in\Txt(L)\).
	Let \(\sigma_0\) be the minimal locking sequence of \(h\) on \(L\). We show
	that  \(h'\) eventually converges to \(h(\sigma_0)\). To that end, let
	\(n_0\in\N\) be large enough such that, with \(D_0=\content(T[n_0])\), we have
	\begin{itemize}
		\item \(\content(\sigma_0)\subseteq D_0\),
		\item \(\sigma_0\leq n_0\) and
		\item for all \(\sigma<\sigma_0\) there exists \(\sigma'\in
		      {(D_0)}_\#^{\leq n_0}\) such that \(h(\sigma)\neq h(\sigma\sigma')\),
		      i.e., \(\sigma'\) witnesses \(\sigma\notin M(D_0,n_0)\).
	\end{itemize}
	Then, for all \(n\geq n_0\), we have \(\min(M(\content(T[n]),n))=\sigma_0\) and thus
	\(h'\) converges to \(h(\sigma_0)\). As this is a correct hypothesis for
	\(L\), \(h'\) learns \(\La\).
\end{proof}

By patching in the information given \cite{KSS17}, even iterative $\Bc$-learners
are as powerful as Gold-style $\Bc$-learners. This also holds true for total
such learners. We provide the next theorem.

\begin{theorem}
	We have that ${[\totalCp\Txt\It\Bc\IndR]}_\Ind = {[\totalCp\Txt\G\Bc\IndR]}_\Ind$.
\end{theorem}

\begin{proof}
    Immediately, we have ${[\totalCp\Txt\It\Bc\IndR]}_\Ind \subseteq
    {[\totalCp\Txt\G\Bc\IndR]}_\Ind$. We apply a padding argument \cite{KSS17}
    for the other direction. By Theorem~\ref{Thm:ind-fam}, it suffices to show that
	\[
		{[\tau(\Cind)\Txt\G\Bc_C]} \subseteq {[\tau(\Cind)\Txt\G\Ex_C]}.
	\]
    Let $h$ be a $\tau(\Cind)\Txt\G\Bc_C$-learner and let $\La\subseteq
    \tau(\Cind)\Txt\G\Bc_C(h)$.
	Recall that $\pad$ is an injective padding function such that for all $e\in\N$ and
	all finite sequences $\sigma$ we have $\varphi_{\pad(e,\sigma)} = \varphi_e$.
	We define the iterative learner $h'$ for all previous hypotheses $p$, all
	finite sequences $\sigma$ and all $x \in \N$,
	\begin{align*}
		h'(\emptyset)          & = \pad(h(\varepsilon) ,\varepsilon);              \\
		h'(\pad(p, \sigma), x) & = h'(\pad(h(\sigma\concat x), \sigma \concat x)).
	\end{align*}
	It is immediate to see that, for all sequences $\sigma$, we have
    $\varphi_{{(h')}^*(\sigma)} = \varphi_{h(\sigma)}$. Thus, $h'$
    $\tau(\Cind)\Txt\It\Bc$-learns~$\La$.
\end{proof}

As patching changes the hypothesis with every new datum, this approach does not
work for explanatory iterative learners. It is known that this problem cannot be
solved as there exists a well-known class separating set-driven explanatory
learners from iterative ones. This result transfers to total learners as well as
the next theorem shows.

\begin{theorem}
	We have that ${[\totalCp\Txt\Sd\Ex\IndR]}_\Ind \setminus
		{[\totalCp\Txt\It\Ex\IndR]}_\Ind\neq\emptyset$.
\end{theorem}

\begin{proof}
    This is a standard proof and we include it for completeness \cite{JORS99}.
    By Theorem~\ref{Thm:ind-fam}, it suffices to provide a class of languages
    $\La$ which is
    $\tau(\Cind)\Txt\Sd\Ex_C$-learnable but not $\tau(\Cind)\Txt\It\Ex_C$ so. We
    define $\La \coloneqq \{ \N \setminus \{ 0 \} \} \cup \{ D \cup \{ 0 \} \mid
    D \subseteq_\Fin \N \}$.
    Then, the following learner learns $\La$. Fix $p_0$ as a code for the
    language $\N \setminus \{ 0 \}$ and define, for any finite sequence
    $\sigma$,
	\[
        h(\sigma) = \begin{cases}
			p_0,                    & \falls 0 \notin \content(\sigma); \\
			\ind(\content(\sigma)), & \sonst.
		\end{cases}
    \]
	It is immediate that $h$ learns $\La$. Assume there exists a learner $h'$
	which $\tau(\Cind)\Txt\It\Ex$-learns $\mathcal{L}$. Let \(L=\N \setminus \set{0} \),
	let $T$ be a text of $L$ and let $n_0\in\N$ such that for all $n \geq n_0$ we
	have $h'(T[n_0]) = h'(T[n])$. Let $x=\max(\content(T[n_0]))$, then
	on the following two texts of distinct languages in $\La$
	\begin{align*}
		T_1=\sigma\concat(x'+1)\concat 0^{\infty}; \\
		T_2=\sigma\concat(x'+2)\concat 0^{\infty},
	\end{align*}
    the learner $h'$ generates the same hypotheses. Thus, it is unable to
    distinguish between these two. Therefore, $\mathcal{L}$ cannot be learned by
    $h'$.
\end{proof}

On the other hand, each iterative learner can be made set-driven by simply,
given all data, mimicking the iterative learner on input with pause-symbols
between each two elements. This well-known approach also works for our setting
as well. The following theorem holds.

\begin{theorem}
    We have that ${[\totalCp\Txt\It\Ex\IndR]}_\Ind \subseteq
    {[\totalCp\Txt\Sd\Ex\IndR]}_\Ind$.
\end{theorem}

\begin{proof}
    This is a standard proof and we include it for completeness \cite{KS95}. Let
    $h$ $\Txt\It\Ex$-learn the indexed family $\La$ with respect to $\Ha$. We
    provide a $\Sd$-learner learning $\La$. To that end, we expand the
    hypothesis space $\Ha$ by adding all finite sets. This new hypothesis space
    we denote by $\Ha'$. For ease of notation, we refer to these new indices as,
    for all $D$, $\ind(D)$. Now, for any set \(D\), let \(\sort_\#(D)\) be the
    sequence of the elements in \(D\) sorted in ascending order, with a \(\#\)
    between each two elements. Furthermore, let \(h^*\) be the starred form of
    \(h\). Now, we define \(h'\) as, for all finite sets \(D\),
	\begin{align*}
		h'(D)=\begin{cases}
			h^*(\sort_\#(D)), & \cIf h^*(\sort_\#(D))=h^*(\sort_\#(D)\concat\#); \\
			\ind(D),          & \otw.
        \end{cases}
	\end{align*}
	To show that \(h'\) learns \(\La\) with respect to $\Ha'$, let \(L\in \La\). If \(L\) is
	finite, then either \(h^*(\sort_\#(L))=h^*(\sort_\#(L)\concat\#)\), in which
	case \(h\) converges to \(h'(L)=h^*(\sort_\#(L))\) on text
	\(\sort_\#(L)\concat\#^\infty\). Otherwise, we have \(h'(L)=\ind(L)\). In both
	cases, \(h'\) learns \(L\) as \(h'(L)\) is a correct hypothesis for~\(L\).

	On the other hand, if \(L\) is infinite, then \(h\) converges to a
	correct hypothesis for \(L\) on the text \(\sort_\#(L)\). Let \(\sigma_0\)
	be the initial sequence of \(\sort_\#(L)\) on which \(h\) has
	converged and let \(D_0=\content(\sigma_0)\). Then, for all \(x\in
	\N\setminus D_0\), we have
	\(h^*({\sigma_0}\concat x)=h^*(\sigma_0)=h^*({\sigma_0}\concat\#)\)  as \(h\) is
	iterative. Therefore, for all \(D'\) with
	\(D_0 \subseteq D' \subseteq L \), we have
	\(h^*(\sort_\#(D'))=h^*(\sort_\#(D')\concat\#)\) and
	\(h^*(\sort_\#(D'))=h^*(\sort_\#(D_0))\), which is a correct hypothesis for
	\(L\). As \(h'(D')=h(\sort_\#(D'))\), we have the convergence of \(h'\) to a
	correct hypothesis for \(L\) and, thus, \(h'\) learns \(L\).
\end{proof}

Interestingly, only iterative learners benefit from loosening the convergence
criterion.  We have already investigated the situation for Gold-style and
partially set-driven learners. Now, we conclude this section by showing that,
first, total set-driven learners and then also transductive ones do not benefit
from this relaxation.

Considering set-driven learners, we first show that behaviorally correct such
learners may be assumed target-cautious in general. We do so by conducting a
forward search, checking the learners output on each possible future hypothesis.
Should we detect inconsistencies, we know that the current information is not
locking and, thus, we can stop the enumeration. This way, no overgeneralization
will happen as, otherwise, locking sets must be included in the search. We
obtain the following result.

\begin{lemma}\label{lem:sdbc-sdcauttarbc}
    We have that ${[\tau(\Cons)\Txt\Sd\CautTar\Bc\IndR]}_{\Ind} =
    {[\totalCp\Txt\Sd\Bc\IndR]}_{\Ind}$.
\end{lemma}

\begin{proof}
	The inclusion
    ${[\tau(\Cons)\Txt\Sd\CautTar\Bc\IndR]}_{\Ind} \subseteq
    {[\totalCp\Txt\Sd\Bc\IndR]}_{\Ind}$ is straightforward. By
    Theorem~\ref{Thm:ind-fam}, it suffices to show
	${[\tau(\Cind)\Txt\Sd\Bc_C]}\subseteq {[\tau(\Cind\Cons)\Txt\Sd\CautTar\Bc_C]}$
    for the other. We apply a similar construction of forwards searches as when
    learning arbitrary classes of languages \cite{DoskocK20}.
	Let $h$ be a total learner with $\La=\tau(\Cind)\Txt\Sd\Bc_C(h)$. According
    to \cref{th:taucons}, we may assume $h$ to be consistent on any input. Now,
    define a $\tau(\Cind\Cons)\Txt\Sd\CautTar\Bc_C$-learner $h'$ as follows.
    Let, for
	all \(x\in\N\) and finite sets \(D\subseteq \N\),
	\begin{align*}
        E(x,D)             & := D \cup \settwo{x} \cup \settwo{ x' \leq x \mid
            \varphi_{h'(D)}(x') = 1}; \\
		\varphi_{h'(D)}(x) & = \begin{cases}
			1, & \falls x \in D;\\
			0, & \sonstfalls \varphi_{h(D)}(x) = 0;\\
            1, & \sonstfalls \forall D'', D \subseteq D'' \subseteq E(x,D)\colon
                E(x,D) \subseteq C_{h(D'')}; \\
			0, & \sonst.
		\end{cases}
	\end{align*}
    Intuitively, the conjecture of \(h'(D)\) contains \(D\) itself and certain
    additional elements of the
    hypothesis of \(h\) on \(D\). For these additional elements, $h'$ checks
    whether all possible future hypotheses of $h$ contain these elements as
    well. If so, $h'$ adds them in its hypothesis, otherwise it does not. This
    way, we prevent overgeneralizing target languages.
    Note that by construction $h'$ is $\tau(\Cind\Cons)$. Furthermore, note
    that, for any finite set $D$, we have
	\begin{align}
		C_{h'(D)} \subseteq C_{h(D)}. \label{Bed:Ccons}
	\end{align}
    We first show that $h'$ \(\Bc\)-learns any language \(L\in\La\). We
    distinguish the following cases.
	\begin{enumerate}
		\itemin{Case 1:} $L$ is finite. Since $h$ learns $L$, we have $\varphi_{h(L)} =
			      \chi_L$. Consider $h'(L)$. Now, for any element $x \in L$, we have
		      $\varphi_{h'(L)}(x)=1$, by definition.
		      For any element $x \notin L$, we have $\varphi_{h(L)}(x) = 0$ and, therefore,
		      $\varphi_{h'(L)}(x) = 0$ as well. Thus, $\varphi_{h'(L)} =\chi_{L}$ and
		      $h'$ learns $L$.
        \itemin{Case 2:} \(L\) is infinite. Let $D_0$ be a $\Bc_C$-locking set
            for $h$ on $L$. We show that, for any $D$ with $D_0 \subseteq D
            \subseteq L$, $h'(D)$ is a correct hypothesis for $L$. We need to
            show that $L = C_{h'(D)}$. Note that, by
            Condition~\eqref{Bed:Ccons}, $C_{h'(D)} \subseteq C_{h(D)} = L$.
            Thus, it remains to be shown that $L \subseteq C_{h'(D)}$. To that
            end, let \(x\in L\). If $x \in D$, then $x \in C_{h'(D)}$ by
            consistency. Otherwise, we have $\varphi_{h(D)}(x) = 1$ and, thus,
            are in the third case of the definition of $h'$. We show that $x$
            gets enumerated this way. By Condition~\eqref{Bed:Ccons}, we get
		    \begin{align*}
			    E(x,D) & = D\cup\set{x}\cup \{ x' \leq x \mid \varphi_{h'(D)}(x')=1 \}\\
			           & \subseteq D\cup\set{x}\cup 
                            \{ x' \leq x \mid \varphi_{h(D)}(x')=1 \} \subseteq L.
		    \end{align*}
              As $D_0$, and therefore also $D$, is a $\Bc_C$-locking set, we
              have for all $D''$ with $D \subseteq D'' \subseteq L$ that
		      \[
			      E(x,D) \subseteq L = C_{h(D'')}.
		      \]
		      So the third condition is met and, therefore, $\varphi_{h'(D)}(x)=1$.
		      Hence, $h'$ \(\Bc_C\)-learns $L$.
	\end{enumerate}
	Finally, it remains to be shown that $h'$ is target-cautious.
    To that end, assume that $h'$ is not target-cautious. Thus, there exists a
    language $L \in \La$ and a set $D \subseteq L$ such that
	\(
	L \subsetneq C_{h'(D)}.
	\)
    Let $\tilde{x} \in C_{h'(D)} \setminus L$ and let $D_0 \supseteq D$ be a
    $\Bc_C$-locking set for $L$ on $h$. Let $x' \coloneqq \max((D_0 \cup \{
    \tilde{x} \}) \setminus D$.
    As $x' \in C_{h'(D)}$ but not in $D$, it must be enumerated by the third
    condition of the definition of $h'$. Note that $D_0 \subseteq E(x',D)$ (as
    $D_0$ must be enumerated until $x'$). Now, for all $D''$ with $D \subseteq
    D'' \subseteq E(x',D)$, it must hold that
	\[
		x \in E(x', D_0) \subseteq C_{h(D'')}.
	\]
    However, this is a contradiction for $D'' = D_0$ as $C_{h(D_0)} = L$ but $x
    \notin L$. This concludes the proof.
\end{proof}

In a second step, we construct an explanatory learner from the target-cautious
behaviorally correct learner. The idea is to always mimic the $\Bc$-learner on
the $\leq$-minimal set on which it is consistent. This way, we obtain syntactic
convergence. On the other hand, the final hypothesis cannot be incorrect as,
eventually, the learner has enough information to figure out incorrect guesses
and, as it is target-cautious, these consistent conjectures are no
overgeneralizations. The following result holds.

\begin{theorem}\label{Thm:B}
	We have that ${[\totalCp\Txt\Sd\Ex\IndR]}_\Ind = {[\totalCp\Txt\Sd\Bc\IndR]}_\Ind$.
\end{theorem}

\begin{proof}
    The inclusion ${[\totalCp\Txt\Sd\Ex\IndR]}_\Ind \subseteq
    {[\totalCp\Txt\Sd\Bc\IndR]}_\Ind$ is immediate. For the other, by
    Theorem~\ref{Thm:ind-fam}, it suffices to show that
	\[
		{[\tau(\Cind)\Txt\Sd\Bc_C]} \subseteq {[\tau(\Cind)\Txt\Sd\Ex_C]}.
	\]
	Let $h$ be a learner and let $\La\subseteq
        \tau(\Cind)\Txt\Sd\Bc_C(h)$. By Theorem~\ref{lem:sdbc-sdcauttarbc}, we
        may assume $h$ to be target-cautious. We provide a learner $h'$ which
        $\Ex_C$-learns $\La$. The main
    idea is to search for the first set on which the learner $h$ is consistent.
    By its target-cautiousness, this way we will, eventually, conjecture the
    right language.  For the formal details, fix a total order $\leq$ on finite
    sets. For any finite set $D$, we define the following auxiliary functions as
	\[
		M(D) =  \{D' \subseteq D \mid \forall x\in D \colon \varphi_{h(D')}(x) = 1 \}.
	\]
    Finally, for any finite set $D$, define $h'(D) = h(\min_\leq (M(D)))$. To
    show correctness, let $L \in \La$. We distinguish the following cases.
	\begin{enumerate}
        \itemin{1. Case:} $L$ is finite. Let $D' \subseteq L$ such that $h'(L) =
            h(D')$. Since $D' \in M(L)$, we have $L \subseteq C_{h(D')}$. Since
            $h$ is target-cautious, the equality holds, that is, $L =
            C_{h(D')}$. Thus, $h'(L) = h(D')$ is a correct hypothesis.
        \itemin{2. Case:} $L$ is infinite. Let $D_0$ be the $\leq$-minimal set
            such that $C_{h(D_0)} = L$. Let $D_1 \supseteq D_0$ such that
            $\min_\leq(M(D_1)) = D_0$. Such a set exists as $h$, due to the
            minimal choice of $D_0$, conjectures incorrect guesses on $D'
            \subseteq L$ with $D' < D$ which do not overgeneralize the target
            language. Then, for any $D$, with $D_1 \subseteq D \subseteq L$, we
            have $h'(D) = h(D_0)$, a correct conjecture. \qedhere
	\end{enumerate}
\end{proof}

Finally, behaviorally correct transductive learners, being unable to save any
information about previous guesses, can be made explanatory immediately. One
simply awaits a non-? guess and then checks for the first element in this guess
which also produces a non-?. We provide the theorem.

\begin{theorem}
	We have that ${[\totalCp\Txt\Td\Ex\IndR]}_\Ind = {[\totalCp\Txt\Td\Bc\IndR]}_\Ind$.
\end{theorem}

\begin{proof}
	The inclusion ${[\totalCp\Txt\Td\Ex\IndR]}_\Ind \subseteq {[\totalCp\Txt\Td\Bc\IndR]}_\Ind$
	follows immediately. For the other, it suffices, by Theorem~\ref{Thm:ind-fam},
	to show that
	\[
		{[\tau(\Cind)\Txt\Td\Bc_C]}\IndF \subseteq {[\tau(\Cind)\Txt\Td\Ex_C]}\IndF.
	\]
	Let $h$ be a $\tau(\Cind)\Txt\Td\Bc_C$-learner and let $\La\subseteq
		\tau(\Cind)\Txt\Td\Bc_C(h)$. We define \(M\in\totalCp\) and \(h'\in\totalCp\)
	such that for all \(x\in\N, y\in\N_\#\),
	\begin{align*}
		M(x)  & = \set{ x' \leq x \mid \varphi_{h(x)}(x') = 1 \wedge h(x')
		    \neq \mbox{?}};\\
		h'(y) & = \begin{cases}
			h(\#),         & \falls y = \#;\\
			\mbox{?} ,     & \sonstfalls h(y) = \mbox{?};\\
			h(\min(M(y))), & \sonst.
		\end{cases}
	\end{align*}
	By construction, \(h'\) is only outputs \(C\)-indices (or \(\mbox{?}\)).
	Intuitively, \(h'\) outputs the hypothesis of \(h\) on the smallest element
	in the hypothesis \(h\) on the current datum (if it is not ``?'').
	We claim that $h'$ learns $\La$. Let $L \in \La$. First, note that for any
	$x\in L$, we have that either $h(x) = \mbox{?}$ or $h(x)$ is a $C$-index of
	$L$. If that were not the case, $h$ would not identify $L$ on any text which
	has infinitely many occurrences of $x$. Furthermore, for at least one $x \in
		L$, $h(x)$ must not be ``?''. Thus, there exists a minimal $x' \in L$, such
	that $h(x')$ is a characteristic index of $L$. The idea of this construction
	is to search for such minimal $x'$. Note that, if \(h(x)\neq\mbox{?}\),
	\(M(x)\neq\emptyset\) as \(x\in M(x)\).

	Let $T \in \Txt(L)$ and let $n_0\in\N$ be minimal such that $\content(T[n_0]) \neq
		\emptyset$ and such that \(h(T(n_0-1))\neq\mbox{?}\). Then, $h'(T(n_0-1)) =
		h(\min(M(T(n_0-1))))$, a correct guess. Furthermore,
	for $n > n_0$ either $h(T(n)) = \mbox{?}$ and with it $h'(T(n)) = \mbox{?}$ or,
	otherwise, $h'(T(n)) = h(\min(M(T(n)))) = h(\min(M(T(n_0-1))))$. Thus, we have
	$\Ex_C$-convergence.
\end{proof}

We remark that all these results, except for Lemma~\ref{lem:sdbc-sdcauttarbc}
and Theorem~\ref{Thm:B}, also hold for (possibly) partial learners. These are
obtained either directly, for example by applying
Theorem~\ref{thm:g-wlog-total}, or by slightly changing the provided proofs.
However, one cannot directly translate Lemma~\ref{lem:sdbc-sdcauttarbc}, and
therefore Theorem~\ref{Thm:B}, as, in the forward search, totality of the
learner is key. Otherwise, this search can be indefinite, breaking the
indexability. We conclude this work by posing the following open question.

\begin{problem}
Does ${[\Txt\Sd\Ex\IndR]}_\Ind = {[\Txt\Sd\Bc\IndR]}_\Ind$ hold?
\end{problem}

\bibliography{LTbib}

\end{document}